\documentclass[ a4paper]{article}
\usepackage[english]{babel}
\usepackage{amsmath, amssymb, amsthm, amsfonts}
\usepackage[utf8x]{inputenc}
\usepackage[all]{xy}
\usepackage{hyperref}
\usepackage[flushleft]{threeparttable}
\usepackage{nicefrac}
\usepackage{marvosym}
\usepackage{multirow}
\usepackage[dvipsnames]{xcolor}
\usepackage{colonequals}
\usepackage[toc,page]{appendix}
\usepackage{subfig}
\usepackage{microtype}
\usepackage{enumerate}
\usepackage{amssymb}
\usepackage{cite}
\usepackage{amsthm}
\usepackage{placeins}
\usepackage{amsmath}
\usepackage{verbatim}
\usepackage{wasysym}
\usepackage{tikz}
\usepackage{wrapfig}
\usepackage{layout}
\usepackage{fancyhdr}
\usepackage{fancyref}
\usepackage{float}
\usepackage{setspace}
\usepackage[figuresright]{rotating}
\usepackage{graphicx}
\usepackage{caption}
\usepackage{algorithm}
\usepackage{algorithmic}
\usepackage{pgf}
\usepackage{dsfont}
\usepackage{textcomp}

\usepackage{marginnote}

\usetikzlibrary{matrix,arrows,decorations.pathmorphing,backgrounds,positioning,fit,petri,automata}

\allowdisplaybreaks

\newcommand*{\NN}{\mathbb{N}}

\newcommand*{\parO}{\partial{\cK}}

\newcommand*{\coA}{\mathfrak{C}_{A}}
\newcommand*{\coAx}{\mathfrak{C}_{Z_{x}}}

\newcommand*{\coBone}{\mathfrak{C}_{B_{1}}}
\newcommand*{\coBP}{\mathfrak{C}_{B_{P}}}
\newcommand*{\coBp}{\mathfrak{C}_{B_{p}}}

\newcommand*{\uD}{\underline{D}}

\newcommand*{\coYROF}{\mathfrak{C}_{Y^{TV-\ell^2}}}

\newcommand*{\ucoD}{\underline{\mathfrak{C}}_{\underline{D}}}

\newcommand*{\coC}{\mathfrak{C}_{C}}

\newcommand*{\ucoBB}{\underline{\underline{\mathfrak{C}}}_{\underline{B}}}
\newcommand*{\ucoYY}{\underline{\underline{\mathfrak{C}}}_{\underline{Y}}}
\newcommand*{\ucoYROF}{\underline{\underline{\mathfrak{C}}}_{Y^{TV-\ell^2}}}

\newcommand*{\ucoR}{\underline{\underline{\mathfrak{C}}}_{\underline{R}}}
\newcommand*{\ucoRT}{\underline{\underline{\mathfrak{C}}}_{\underline{\widetilde{R}}}}

\newcommand*{\uH}{\underline{H}}
\newcommand*{\uR}{\underline{R}}

\newcommand*{\coV}{\mathfrak{C}_{V}}
\newcommand*{\coVT}{\mathfrak{C}_{\widetilde{V}}}

\newcommand*{\coM}{\mathfrak{C}_{M}}

\newcommand*{\coY}{\mathfrak{C}_{Y}}

\newcommand*{\coF}{\mathfrak{F}}

\newcommand*{\coE}{\mathfrak{E}}

\newcommand*{\ucoB}{\underline{\mathfrak{C}}_{\underline{B}}}

\newcommand*{\ucoBT}{\underline{\mathfrak{C}}_{\underline{\widetilde{B}}}}

\newcommand*{\inp}{(\underline{B},\underline{\widetilde{B}})}

\newcommand*{\Usol}{U^{\dagger}}
\newcommand*{\Wsol}{\underline{W}^{\dagger}}
\newcommand*{\lamsol}{\underline{\lambda}^{\dagger}}

\newcommand*{\uW}{\underline{W}}

\newcommand*{\uB}{\underline{B}}
\newcommand*{\uY}{\underline{Y}}
\newcommand*{\uBT}{\underline{\widetilde{B}}}

\newcommand*{\ulam}{\underline{\lambda}}

\newcommand*{\fuJ}{\mathcal{J}}

\newcommand*{\fuY}{\mathsf{Y}}

\newcommand*{\fuSh}{\Omega_{S}}
\newcommand*{\fuIm}{\Omega_{F}}

\newcommand*{\fus}{\mathsf{s}}

\newcommand*{\fuS}{\mathsf{S}}

\newcommand*{\R}{\mathbb{R}}
\newcommand*{\N}{\mathbb{N}}
\newcommand*{\C}{\mathbb{C}}

\newcommand*{\Z}{\mathbb{Z}}

\newcommand*{\fix}{(U^{\dagger},\underline{W}^{\dagger},\underline{\lambda}^{\dagger})}

\newcommand*{\GammaP}{\Gamma_{P}}

\newcommand*{\lb}{\left(}
\newcommand*{\rb}{\right)}
\newcommand*{\lbr}{\left[}
\newcommand*{\rbr}{\right]}
\newcommand*{\lbb}{\left\{}
\newcommand*{\rbb}{\right\}}
\newcommand*{\lla}{\left\langle}
\newcommand*{\rra}{\right\rangle}
\newcommand*{\lab}{\left|}
\newcommand*{\rab}{\right|}

\newcommand*{\SpM}{\R^{n\times{}m}}
\newcommand*{\SpMC}{\C^{n\times{}m}}

\newcommand*{\allkl}{0\leq{}k\leq{}n-1\text{ and }0\leq{}\ell\leq{}m-1}

\newcommand*{\cK}{\mathcal{K}}

\numberwithin{equation}{section}

\theoremstyle{plain} 
\newtheorem{The}{Theorem}[section]

\newtheorem{Lem}[The]{Lemma}
\newtheorem{Cor}[The]{Corollary}
\newtheorem{Rem}[The]{Remark}

\theoremstyle{definition}
\newtheorem{Def}[The]{Definition}

\newtheorem{PropDef}[The]{Proposition and Definition}

\newtheorem{Conv}[The]{Convention}

\newtheorem{Prob}[The]{Problem}
\newtheorem{Alg}[The]{Algorithm}

\begin{document}

\title{Generalized Intersection Algorithms with Fixpoints for Image Decomposition Learning} 

\author{Robin Richter\thanks{R. Richter and S.F. Huckemann are with the Felix-Bernstein-Institute for Mathematical Statistics in the Biosciences at the University of Goettingen, 37077 Goettingen, Germany (e-mails: robin.richter@mathematik.uni-goettingen.de, huckeman@math.uni-goettingen.de)}\,,  Duy H. Thai\thanks{D.H. Thai is with the Department of Mathematics at Colorado State University}\, and Stephan F. Huckemann$^{\ast}$}

\maketitle

\begin{abstract}
In image processing, classical methods minimize a suitable functional that balances between computational feasibility (convexity of the functional is ideal) and suitable penalties reflecting the desired image decomposition. The fact that algorithms derived from such minimization problems can be used to construct (deep) learning architectures has spurred the development of  
algorithms that can be trained for a specifically desired image decomposition, e.g. into cartoon and texture. While many such methods are very successful, theoretical guarantees are only scarcely available. To this end, in this contribution, we formalize a general class of intersection point problems encompassing a wide range of (learned) image decomposition models, and we give an existence result for a large subclass of such problems, i.e. giving the existence of a fixpoint of the corresponding algorithm. This class generalizes classical model-based variational problems, such as the TV-$\ell^2$-model or the more general TV-Hilbert model. To illustrate the potential for learned algorithms, novel (non learned) choices within our class show comparable results in denoising and texture removal.
\end{abstract}

\section{Introduction}

Decomposing an image into parts that carry desired information and other parts that are considered as nuisance is an old, and yet, an ever new problem. For decades this problem has been attacked under a variational paradigm crafting suitable objective functions that reflect a desired image decomposition, e.g. \cite{Sch09,CCN15}. Often, however, it is not fully clear, how optimal these objective functions are for a given task. Typically for cartoon-texture decomposition, there are three goals: identifying piecewise constant parts, keeping contrast and avoiding artefacts. For this and other tasks such as denoising/deblurring or classification, see e.g. \cite{KBPS11,BP10,SB07,MMG12,Ber06,YGO05}.

Exploiting recently developed 
(convolutional) neural networks, improved objective functions have been \emph{learned}, 
e.g. \cite{MJU17,Wan16}. 
Also, 
since solutions of minimization problems are often obtained algorithmically, and every iteration step can be viewed as applying a family of convolution filters, suitable filters have been learned (e.g. \cite{CYP15,CP17,AO17,AO18,YSLX16,LCG18}). Corresponding solutions are often no longer minimizers of an (explicitly  given) objective function, e.g. \cite{YSLX16, LCG18}.
While, corresponding algorithms come naturally with an explicit \emph{intersection problem} -- the solution of which are fixed points -- as the connection with an objective function may be lost, it is a priori unclear whether fixed points exist at all. 

In our contribution we prove that fixed points exist for a large class of intersection problems. This class encompasses many classical models such as the TV-$\ell^2$-model or the more general TV-Hilbert model, but is much larger. For a family of (convolution) filters that is not learned -- applying our result to learned and learning filters is the subject of ongoing research -- 
we illustrate that keeping contrast and avoiding artefacts can be achieved by convolution filters that do not correspond, to the best of the authors' knowledge, to a minimization problem.

Throughout this contribution we consider
\begin{itemize}
 \item an $n\times m$ pixel image given by a gray value matrix $F \in \SpM$,
 \item a filter matrix $A \in \SpM$ and families of matrix filters $\uB,\uBT \in(\SpM)^{\times P}$ of fixed length $P\in \NN$
 inducing \emph{discrete convolutional operators} $\coA$ and $\ucoB,\ucoBT$, respectively, with adjoints $\coA^{\ast}$, etc., as detailed in Convention \ref{con}.
 \item \emph{shrinkage functions} $\fuS_{\kappa}$ ($\kappa =1,2$) and norms $|\cdot|_{1,\kappa}$ for $(\SpM)^{\times P}$, also detailed in Convention \ref{con}.,
 \item the desired output matrix $U\in\SpM$, which in the context of this contribution is a denoised image (cf. Section \ref{sec:denoising}) or a \emph{cartoon} (cf. Section \ref{sec:cartoon}), 
 \end{itemize}
and the intersection  Problem \ref{prob:gen_intro} below. 
In this and Algorithm \ref{alg:Generalized_Algorithm} further below the superscript ``$G$'' stands for the generalization from the literature (e.g.~\cite{WT10}) and the subscript ``$C$'' for the constraining condition and step, respectively (see (\ref{eq:originalprobW} in Section~\ref{sec:previous}).

\begin{Prob}\label{prob:gen_intro}
 Find $\fix\in \Omega^{\kappa}_{1}\cap \Omega_{2}^{G}\cap \Omega_{C} \subset (\SpM)^{\times(1+2P)}$ where 
\begin{equation}\label{eq:gen_intersection_intro}
\begin{aligned}
\Omega^{\kappa}_{1}&:=\lbb{}\lb{}U,\uW,\ulam\rb{}\in(\SpM)^{\times{}(1+2P)}:{}\uW{}=\fuS_{\kappa}\lb\ucoB\lb{}U\rb-\frac{1}{\beta}\ulam;\,\frac{1}{\beta}\rb\rbb\,,\\
\Omega_{2}^{G}&:=\lbb{}\lb{}U,\uW,\ulam\rb{}\in(\SpM)^{\times{}(1+2P)}:{}U{}=\coA\lb{}F\rb+\ucoBT^{\ast}\lb{}\uW+\frac{1}{\beta}\ulam\rb\rbb\,,\\
\Omega_{C}&:=\lbb(U,\uW,\ulam)\in(\SpM)^{\times{}(1+2P)}:\ucoB\lb{}U\rb=\uW\rbb\,.
\end{aligned}
\end{equation}
\end{Prob}
In order to solve Problem \ref{prob:gen_intro}, we formalize a generalization of an \emph{augmented Lagrangian/alternating method of multipliers} (AL/ADMM) algorithm originally introduced for problems of the form (\ref{eq:variatonalproblems_L1_one}) below, by~\cite{EB92,Gab83}, which, in our general setting rewrites as follows.

\begin{Alg}\label{alg:Generalized_Algorithm}
For $\tau = 1,2,\ldots$, do until 
$\frac{\lab\lab{}U^{(\tau)}-U^{(\tau-1)}\rab\rab}{\lab\lab{}U^{(\tau-1)}\rab\rab}<\epsilon $
\begin{equation}\label{eq:algo_intro}
\begin{aligned}
\uW^{(\tau)}&=\fuS_{\kappa}\lb{}\ucoB\lb{}U^{(\tau-1)}\rb-\frac{1}{\beta}\ulam^{(\tau)};\,\frac{1}{\beta}\rb\,,\\
U^{(\tau)}&=\coA{}(F)+\ucoBT^{\ast}\lb{}\uW^{(\tau)}+\frac{1}{\beta}\ulam^{(\tau)}\rb,\\
\ulam^{(\tau+1)}&=\ulam^{(\tau)}+\beta\lb\uW^{(\tau)}-\ucoB(U^{(\tau)})\rb\,.
\end{aligned}
\end{equation}
\end{Alg}

We have at once the following equivalence.

\begin{Rem}\label{rem:equivalence}
Every intersection point of Problem~\ref{prob:gen_intro} is a fixed point of Algorithm~\ref{alg:Generalized_Algorithm} and vice versa.
\end{Rem}


For the 
class of $(A,\uB,\uBT) \in (\SpM)^{\times (1+2P)}$ with $(\uB,\uBT)$ \emph{weakly factoring} -- this new concept is defined in Definition \ref{def:generalizations} -- and satisfying a contraction and positive semi-definite condition we give in Theorem \ref{the:existenceouter} guarantees for existence of solutions of Problem \ref{prob:gen_intro}. 

As insinuated above, Problem \ref{prob:gen_intro} is motivated by its applicability to a range of algorithms learning $\uB$ and $\uBT$, in particular it covers the ADMM-NET of \cite{YSLX16,LCG18}. 
In order to assess the success of such algorithms in applications (\cite{MJU17} survey such considerations) we identify special cases of Problem~\ref{prob:gen_intro}. 

A number of learning problems (\cite{CYP15,CP17,AO17,AO18,YSLX16,LCG18}) build on minimizing special cases of the classical $\ell^{1}$-regularized functional
\begin{equation}\label{eq:variatonalproblems_L1_one}
\mathcal{J}_{\text{VAR}}(U)= \lab{}\ucoB(U)\rab_{1,\kappa}+\frac{\mu}{2}\lab\lab{}U-F\rab\rab^{2}\,
\end{equation}
(cf. \cite{Sch09}) with $\mu>0$. This is a special case of Problem \ref{prob:gen_intro}, as we will see in Section \ref{scn:l1-reg}.  
Notably, plugging in the two-dimensional discrete gradient $\nabla$ - which, considering periodic boundary conditions, can be written as a convolution operator - for $\ucoB$ ($P=2$) and $\kappa =2$ in (\ref{eq:variatonalproblems_L1_one}) gives the classic TV-$\ell^2$ minimization problem 
\begin{equation}\label{prob:ROF} \mathcal{J}_{\text{TV-$\ell^2$}}(U)= \lab{}\nabla U\rab_{1,2}+\frac{\mu}{2}\lab\lab{}U-F\rab\rab^{2}\,,\end{equation}
with its continuous version proposed by \cite{ROF92}. 

Closely  related is the \emph{anisotropic diffusion process}, cf. \cite{Stei04}. For general choices of $\uB$ and $\uBT$, however, it seems that Problem~\ref{prob:gen_intro} cannot be written as a minimization problem such as (\ref{eq:variatonalproblems_L1_one}).


Aujol and Gilboa~\cite{AG06} have introduced the $TV$-Hilbert problem by applying a convolutional operator $\coM$ to the argument of the $\ell^2$ norm in (\ref{prob:ROF}). 
As another contribution, in Theorem \ref{the:TV-Hilbert}, we show that for suitably chosen $A$ and $\uBT$, the \emph{generalized Hilbert problem}, minimizing
\begin{equation}\label{eq:Hilbert}
\mathcal{J}_{\text{HIL}}(U)=\lab{}\ucoB\lb{}U\rb\rab_{1,\kappa}+\frac{\mu}{2}\lab\lab{}\coM(U-F)\rab\rab^{2}\,,
\end{equation}
with suitable  $M\in\SpM$, where its circular Fourier transform satisfies in particular $\widehat{M}\in\SpM_{+}$, is another special case of Problem~\ref{prob:gen_intro}. In this context we note the introduction of the $G$-norm by Meyer~\cite{Mey01} to model texture components. Since the $G$-norm is hard to compute, it has been approximated by suitably choosing $\uB$ and $M$ in (\ref{eq:Hilbert}), cf. \cite{VO03,VO04,OSV03,GLMV07}.

Our paper is structured as follows. In Section~\ref{sec:two} we introduce the discrete notation and state the existence theorem guaranteeing solutions of Problem~\ref{prob:gen_intro} for  $(A,\uB,\uBT)\in(\SpM)^{\times{}1+2P}$ satisfying rather broad conditions. Its proof is deferred to Section~\ref{sec:four} with some details further deferred to the appendix. In Section~\ref{sec:previous} we link the generalized Problem~\ref{prob:gen_intro} to existing minimization problems.
In fact, our Algorithm \ref{alg:Generalized_Algorithm} is motivated by the AL/ADMM algorithm (Algorithm \ref{alg:original_AL_ROF}, cf. \cite{WT10,BV04}) which solves the saddle point problem  (\ref{eq:saddle}), which is equivalent to (\ref{eq:variatonalproblems_L1_one}).
Moreover
the $TV$-Hilbert model of~\cite{AG06} is reverse-engineered under a condition on $(A,\uB,\uBT)\in(\SpM)^{\times{}1+2P}$ stricter than needed for the existence result. In Section~\ref{sec:five} applications are presented in the context of denoising and the separation of cartoon and texture. We conclude the paper with an outlook to filter learning.

The results presented are taken from the Ph.D thesis of the first author \cite{Ric19b}. 

\section{Conventions and Existence Result}\label{sec:two}
Although many 
approaches for image decomposition are derived in the continuous domain, 
we consider images as matrices, what they are in applications. Similarly, operators on images are computationally discretized into (families of) matrices. 
Particularly handy are linear operators that can be represented via multiple circular matrix convolutions, such as the discrete gradient with periodic boundary conditions or suitably redundant, discrete (wavelet) frame operators (\cite{Mal02}), since they are easy to implement as entry-wise multiplications in the frequency domain, making vectorization redundant. 
 
\begin{Conv}\label{con}

For positive integers $n,m,P$, let  in the following $A,G\in\SpM$ be matrices and $\uB = (B_1,\ldots,B_P), \uH = (H_1,\ldots,H_P)  \in (\SpM)^P =: \Gamma_{P}$ matrix-families.

\begin{enumerate}[(i)]
\item Denote the entries of $A$ by $A[k,\ell]$ for $0\leq{}k\leq{}n-1$ and $0\leq{}\ell\leq{}m-1$. Matrix entries are indexed in $\Z/n\Z\times\Z/m\Z$ and any matrix index exceeding the range of $\lbb{}0,1,\dots{}n-1\rbb\times\lbb{}0,1,\dots{},m-1\rbb$ takes its value modulo $(n,m)$, encoding periodic boundary conditions. \\
$A^{\ast}$ denotes the adjoint of $A$, i.e. its transposed and complex conjugate, and 
$$\widehat{A} = \coF(A) = \left({\rm trace}(F_{rs} A^T)\right)_{r=0,s=0}^{n-1,m-1}$$ 
its discrete circular Fourier transform, where 
\begin{displaymath}
F_{r,s}[k,\ell]=e^{-\frac{2\pi{}ikr}{n}-\frac{2\pi{}i\ell{}s}{m}}\,, 0\leq{}r,k\leq{}n-1\,,0\leq{}s,\ell\leq{}m-1\,.
\end{displaymath} 

\item The linear operator $\coA:\SpM\to\SpM$ denotes the \emph{discrete circular convolution}, or short \emph{matrix convolution}, defined for $G\in\SpM$ by
	\begin{displaymath}\label{eq:convolution}
	\begin{aligned}	\coA(G)=\lb{}\sum_{k=0}^{n-1}
	\sum_{\ell=0}^{m-1}A[k-r,\ell-s]G[k,\ell]\rb_{r=0,s=0}^{n-1,m-1}\,.
	\end{aligned}
	\end{displaymath}
	With the component-wise product $\odot$ we have the well known
	$$ \coA(G) = \coF^{-1} (\widehat A \odot \widehat G)$$
	where $\coF^{-1}(G) = \frac{1}{nm} \left({\rm trace}(\overline{F_{rs}} G^T)\right)_{r=0,s=0}^{n-1,m-1}$. 
	
    Further, $\ucoB:\SpM\to\GammaP$ denotes the \emph{matrix-family convolution}, defined 
    by
	\begin{displaymath}
	\ucoB\lb{}G\rb=\lb\coBone(G),\dots{},\coBP(G)\rb\,,
	\end{displaymath}
	and 
    $\ucoBB:\GammaP\to\GammaP$ denotes the \emph{matrix-family convolution}, defined 
    by
	\begin{displaymath}
	\ucoBB\lb{}\uH\rb=\lb\coBone(H_1),\dots{},\coBP(H_P)\rb\,
	\end{displaymath}
    Moreover, $\coE:\SpM\to\SpM$ is the identity operator.
\item By $\lla\cdot{},\cdot{}\rra$ and $\lab\lab{}\cdot{}\rab\rab$ denote the usual Euclidean inner product and norm, respectively, for vectors, matrices and families of matrices, for instance
        $$ \langle A,G\rangle = {\rm trace}(AG^*)\mbox{ and }\langle \uB,\uH\rangle =  \sum_{p=1}^P{\rm trace}(B_pH_p^*) \,.$$
        These inner products give rise to adjoint operators
		$$ \coA^\ast: \SpM \to \SpM,~\ucoB^\ast : \GammaP \to \SpM\mbox{ and } \ucoBB^\ast : \GammaP\to \GammaP\,.$$
		Verify that $ \coA^\ast = \coC$ with $C[k,\ell] = \overline{A[-k,-\ell]}$ and $\widehat C[k,\ell] = \overline{\widehat A[k,\ell]}$ for $0\leq{}k\leq{}n-1$ and $0\leq{}\ell\leq{}m-1$,
		$$\ucoB^\ast(\uH) = \sum_{p=1}^P \coBp^\ast(H_p) \mbox{ and } \ucoBB^\ast(\uH) = \left(\coBone^\ast(H_1),\ldots,\coBP^\ast(H_P)\right)\,.$$
	    
\item Introduce more norms by
	\begin{eqnarray*}
	\lab\lab\uB\rab\rab_{1,2}&:=&
	\sum_{k=0}^{n-1}\sum_{\ell=0}^{m-1}\lab\lab{}\lb{}B_{1}[k,\ell],\dots{},B_{P}[k,\ell]\rb\rab\rab\,,\\
		\lab\lab\uB\rab\rab_{1,1}&:=&\sum_{p=1}^{P}\sum_{k=0}^{n-1}\sum_{\ell=0}^{m-1}\lab{}B_{p}[k,\ell]\rab{}\,.
		\end{eqnarray*}
		These are related to the \emph{isotropic $\ell^{1}$-norm} and  the \emph{anisotropic $\ell^{1}$-norm}, respectively, from \cite{CDOS12} in the continuous setting. 
		
\item For $\beta\in\R_{+}$ define the \emph{anisotropic soft-shrinkage function} $\fuS_{1}:\Gamma_{P}\times\R_{+}\to\GammaP$ by
\begin{equation*}\label{eq:fuS_A}
\begin{aligned}
\fuS_{1}(\uB,\beta)&:=\lb\lb\fus_{1}\lb{}B_{p}[k,\ell],\beta\rb{}\rb_{k=0,\ell=0}^{n-1,m-1}\rb_{1\leq{}p\leq{}P}\,,\\
\text{ with }\quad{}\fus_{1}(x,\beta)&:=\begin{cases}
x-\beta & \text{if }x>\beta\\
0	&	\text{if }x\in\lbr{}-\beta,\beta\rbr{}\\
x+\beta & \text{if }x<-\beta
\end{cases}\,,
\end{aligned}
\end{equation*} 
and the \emph{isotropic soft-shrinkage function} $\fuS_{2}:\GammaP\times\R_{+}\to\GammaP$ by
\begin{equation*}\label{eq:fuS_B}
\begin{aligned}
\fuS_{2}(\uB,\beta)&:=\lb\lb\fus_{2}\lb\lb\lb{}B_{p}[k,\ell]\rb_{p=1}^{P}\rb^{T},\beta\rb{}[p]\rb_{k=0,\ell=0}^{n-1,m-1}\rb_{1\leq{}p\leq{}P}\,,
\end{aligned}
\end{equation*}
with
\begin{displaymath}
\fus_{2}(x,\beta):=\frac{x}{\lab\lab{}x\rab\rab}\text{max}\lb{}0,\lab\lab{}x\rab\rab-\beta\rb\,
\end{displaymath}
for $x \in \mathbb R^P$.

\end{enumerate}

\end{Conv}

The following definition gives a class of input filters for which Theorem \ref{the:existenceouter} below asserts that Problem \ref{prob:gen_intro} features an intersection point and equivalently Algorithm \ref{alg:Generalized_Algorithm} a fixed point (cf. Remark \ref{rem:equivalence}). It also gives a smaller class for which the minimization Problem  (\ref{eq:Hilbert}) can be reverse engineered as an instance of Problem \ref{prob:gen_intro} in Theorem \ref{the:TV-Hilbert} of Section \ref{scn:Hilbert}.  

\begin{Def}\label{def:generalizations}
We call a triple $(A,\uB,\uBT)$ \textit{input filters} for Algorithm~\ref{alg:Generalized_Algorithm} if $A\in\SpM$ and $\underline{B}=(B_{p})_{p=1}^{P}, \underline{\widetilde{B}}=(\widetilde{B}_{p})_{p=1}^{P}\in\GammaP$. We then say that $(\uB,\uBT)$    
\begin{enumerate}[(i)]
\item 
\textit{weakly factor} if there is $\underline{Y}=\lb{}Y_{p}\rb_{p=1}^{P}\in\Gamma_{P}$, called a \emph{weak factor} of $(\uB,\uBT)$, such that $\widehat{Y}_{p}\in\R_{+}^{n\times{}m}$ and, for all $1\leq{}p\leq{}P$, $\allkl$,
\begin{displaymath}
\widehat{\widetilde{B}}_{p}[k,\ell]=\widehat{Y}_{p}[k,\ell]\widehat{B}_{p}[k,\ell]\,;
\end{displaymath}

\item  \textit{strongly factor} if there is $Y \in \SpM$ called a \emph{strong factor} of $(\uB,\uBT)$, such that $\widehat{Y}\in\SpM_{+}$ and, for all $1\leq{}p\leq{}P$, $\allkl$, 
\begin{displaymath}
\widehat{\widetilde{B}}_{p}[k,\ell]=\widehat{Y}[k,\ell]\widehat{B}_{p}[k,\ell]\,;
\end{displaymath}

\item 
\emph{satisfy the non-expansive and positive semi-definite condition} (NEPC) if
\begin{equation}\label{eq:NEPC}
0\leq{}\,\sum_{p=1}^{P}\overline{\widehat{\widetilde{B}}_{p}[k,\ell]}\widehat{B}_{p}[k,\ell]\,\leq{}1\,,\quad{}\text{ for all }\allkl\,;\tag{NEPC}
\end{equation}

\item 
\emph{satisfy the contraction and positive semi-definite condition} (CPC) if 
\begin{equation}\label{eq:CPC}
0\leq{}\,\sum_{p=1}^{P}\overline{\widehat{\widetilde{B}}_{p}[k,\ell]}\widehat{B}_{p}[k,\ell]\,<1\,,\quad{}\text{ for all }\allkl\,.\tag{CPC}
\end{equation}

\end{enumerate}
\end{Def}

\begin{Rem}\label{ast-factor:eq}
If $Y$ is a strong factor for $(\uB,\uBT)$,  then $\ucoBT = \ucoB \coY$, yielding in conjunction with Convention \ref{con}, (iii),
\begin{eqnarray*} 
    \ucoBT^\ast &=& \coY \ucoB^\ast\,.
\end{eqnarray*}
Similarly, if $\underline{Y}$ is a weak factor for $(\uB,\uBT)$, then $\ucoBT =  \ucoYY\ucoB$, yielding 
\begin{eqnarray*} 
    \ucoBT^\ast &=& \ucoB^\ast \ucoYY \,.
\end{eqnarray*}

\end{Rem}

While  conditions~(\ref{eq:NEPC}) and~(\ref{eq:CPC}) 
demand in particular for all $k,\ell$  that 
$$\sum_{p=1}^{P}\overline{\widehat{\widetilde{B}}_{p}[k,\ell]}\widehat{B}_{p}[k,\ell]\in \R\,,$$
this is a straightforward consequence of weakly factoring filters. The following central lemma elucidates more consequences.

\begin{Lem}\label{lem:eigenvalues}
Let $\inp\in\Gamma_{2P}$ be 
weakly factoring. 
Then, $\ucoBT^{\ast}\ucoB$ diagonalizes and the set of non-zero eigenvalues of 
\begin{displaymath}
\ucoBT^{\ast}\ucoB:\SpM\to\SpM\,\text{ and }\,\ucoB\ucoBT^{\ast}:\GammaP\to\GammaP
\end{displaymath} 
coincide and is given by the following set of real, positive numbers
\begin{displaymath}
\lbb{}\sum_{p=1}^{P}\lb\overline{\widehat{\widetilde{B}}_{p}[k,\ell]}\widehat{B}_{p}[k,\ell]\rb:\quad{}\allkl\rbb\setminus\lbb{}0\rbb\,.
\end{displaymath}

 
\end{Lem}

\begin{proof}

Recall the family $F_{r,s}$ ($0\leq r\leq n-1$, $0\leq s \leq m-1$) of matrices from Convention \ref{con} conveying the discrete Fourier transformation. In consequence of $\widehat{\overline{F}_{r,s}} =E_{r,s}$ with  $E_{r,s}[k,\ell]=\delta_{rk}\delta_{s\ell}$ ($0\leq{}r,k\leq{}n-1$, $0\leq{}s,\ell\leq{}m-1$), we have
\begin{equation*} 
\begin{aligned}
\ucoBT^{\ast}\ucoB\lb{}\overline{F}_{r,s}\rb&=\ucoBT^{\ast}\lb\coBone\lb{}\overline{F}_{r,s}\rb,\dots{},\coBP\lb{}\overline{F}_{r,s}\rb\rb\\
&=\ucoBT^{\ast}\lb{}\frac{1}{nm}\,\coF^{-1}\lb{}\widehat{B}_{1}\odot{}E_{r,s}\rb,\dots,\frac{1}{nm}\,\coF^{-1}\lb{}\widehat{B}_{P}\odot{}E_{r,s}\rb\rb\\
&=\ucoBT^{\ast}\lb{}\frac{1}{nm}\,\coF^{-1}\lb{}\widehat{B}_{1}[r,s]E_{r,s}\rb,\dots,\frac{1}{nm}\,\coF^{-1}\lb{}\widehat{B}_{P}[r,s]E_{r,s}\rb\rb\\
&=\frac{1}{nm}\,\coF^{-1}\lb\sum_{p=1}^{P}\overline{\widehat{\widetilde{B}}_{p}[r,s]}\lb\widehat{B}_{p}[r,s]
E_{r,s}\rb\rb\\
&=
\lb\sum_{p=1}^{P}\overline{\widehat{\widetilde{B}}_{p}[r,s]}\widehat{B}_{p}[r,s]\rb{}\overline{F}_{r,s}\,.
\end{aligned}
\end{equation*}
Moreover, since the $\overline{F}_{r,s}$  ($0\leq{}r,k\leq{}n-1$, $0\leq{}s,\ell\leq{}m-1$) span $\SpM$, $\ucoBT^{\ast}\ucoB$ diagonalizes and all of its eigenvalues are given by
\begin{equation}\label{eq:eigen}
\lbb{}\sum_{p=1}^{P}\lb\overline{\widehat{\widetilde{B}}_{p}[k,\ell]}\widehat{B}_{p}[k,\ell]\rb:\quad\allkl\rbb{}\,.
\end{equation}

Further, let $\underline{Y}\in\GammaP$ be the weak factor 
of $\inp$ 
such that
\begin{displaymath}
\sum_{p=1}^{P}\overline{\widehat{\widetilde{B}}_{p}[k,\ell]}\widehat{B}_{p}[k,\ell]=\sum_{p=1}^{P}\widehat{Y}_{p}[k,\ell]\lab\lab{}\widehat{B}_{p}[k,\ell]\rab\rab^{2}\in\R_{+}\,,
\end{displaymath}
and consider an eigenvalue $\nu\neq{}0$ of $\ucoBT^{\ast}\ucoB$ to an eigenmatrix $U\in\SpMC$.  Then  $U\notin\text{ker}\lb\ucoB\rb$, hence $\ucoB\lb{}U\rb$ is a family of matrices, not all of which are the $0$-matrix, and
\begin{displaymath}
\ucoB\ucoBT^{\ast}\lb\ucoB\lb{}U\rb\rb=\ucoB\ucoBT^{\ast}\ucoB\lb{}U\rb=\ucoB\lb
\nu{}U\rb=\nu\ucoB\lb{}U\rb\,,
\end{displaymath}
making $\nu$ an eigenvalue of $\ucoB\ucoBT^{\ast}$ to the eigenfamily of matrices $\ucoB\lb{}U\rb$.

Vice versa if $\nu\neq{}0$ is an eigenvalue of $\ucoB\ucoBT^{\ast}$ to an eigenfamily of matrices $\uW\in\Gamma$, then in particular $\uW\notin\text{ker}\lb\ucoBT^{\ast}\rb$, hence $\ucoBT^{\ast}(\uW)$ is not the 0-matrix, and
\begin{displaymath}
\ucoBT^{\ast}\ucoB\lb\ucoBT^{\ast}\lb{}\uW\rb\rb=\ucoBT^{\ast}\ucoB\ucoBT^{\ast}\lb{}\uW\rb
=\ucoBT^{\ast}\lb\nu\uW\rb=\nu\ucoBT^{\ast}\lb{}\uW\rb\,,
\end{displaymath}
making $\nu$ an eigenvalue of $\ucoBT^{\ast}\ucoB$ to an eigenmatrix $\ucoBT^{\ast}\lb\uW\rb$. In consequence this shows that all non-zero eigenvalues of $\ucoB\ucoBT^{\ast}$ are given by (\ref{eq:eigen}), concluding the proof.


\end{proof}

The proof of the following main theorem is postponed to Section~\ref{sec:four}.

\begin{The}\label{the:existenceouter}
Let $F\in\SpM,\kappa\in\lbb{}1,2\rbb,\beta\in\R_{+}$ and let $(A,\uB,\uBT)\in\Gamma_{1+2P}$ be input filters with weakly factoring $(\uB,\uBT)$ satisfying the~(\ref{eq:CPC}). Then, Problem~\ref{prob:gen_intro} has a solution $(\Usol,\Wsol,\lamsol)\in\Gamma_{1+2P}$.
\end{The}

\section{Previous Models Which Are Special Cases}\label{sec:previous}

In this section, first, for convenience, we briefly recall the connection between the classic $\ell^{1}$-regularized minimization problem (\ref{eq:variatonalproblems_L1_one}) and the saddle point problem for the augmented Lagrangian 
which gave rise to the the AL/ADMM algorithm, on the one hand and our generalized AL/ADMM Algorithm \ref{alg:Generalized_Algorithm} on the other hand. 

Secondly, we reverse engineer the minimization problem (\ref{eq:Hilbert}), showing that it is a special case of Problem~\ref{prob:gen_intro}.

\subsection{Classic $\ell^{1}$-Regularizations}\label{scn:l1-reg}

Minimizing $\mathcal{J}_{\text{VAR}}$ from (\ref{eq:variatonalproblems_L1_one}) via differentiation poses a problem due to non-differentiability of the $\ell^{1}$-norm. As a workaround (see e.g. \cite{WT10}) an equivalent constrained problem minimizing
\begin{equation}\label{eq:originalprobW}
\widetilde{\mathcal{J}}_{\text{VAR}}(U,\uW)=\lab{}\uW\rab_{1,\kappa}+\frac{\mu}{2}\lab\lab{}U-F\rab\rab^{2}\,,\text{under the constraint}\quad{}\ucoB\lb{}U\rb=\uW\,,
\end{equation}
is considered. To solve~(\ref{eq:originalprobW}),  for $\beta\in\R_{+}$ the augmented Lagrangian method computes the saddle point of the associated augmented Lagrangian functional,
\begin{equation}\label{eq:ROF_augmented_Lagrangian}
\mathcal{J}_{\text{AL}}(U,\uW,\ulam)=\lab{}\uW\rab_{1,\kappa}+\frac{\mu}{2}||U-F||^{2}+\frac{\beta}{2}||\uW-\ucoB\lb{}U\rb||^{2}+\langle\ulam,\uW-\ucoB\lb{}U\rb\rangle\,,
\end{equation} 
i.e. the point $\fix\in\Gamma_{1+2P}$ satisfying  
\begin{equation}\label{eq:saddle}
\mathcal{J}_{\text{AL}}\lb{}\Usol,\Wsol,\ulam\rb\leq{}\mathcal{J}_{\text{AL}}\lb{}\Usol,\Wsol,\lamsol\rb\leq{}\mathcal{J}_{\text{AL}}\lb{}U,\uW,\lamsol\rb\,,
\end{equation}
for all $\lb{}U,\uW,\ulam\rb\in\Gamma_{1+2P}$. Notably, the saddle point is independent of the choice of $\beta\in\R_{+}$.

The AL/ADMM-algorithm computing the saddle point of~(\ref{eq:ROF_augmented_Lagrangian}) iterates the following steps: Minimize $\mathcal{J}_{\text{AL}}$ over $\uW$ and $U$, for fixed $(U,\ulam)$ and $(\uW,\ulam)$, respectively, and update $\ulam$ via a gradient step. 

\begin{Alg}\label{alg:original_AL_ROF}
For $\tau=1,2,\dots{}$ do 
\begin{displaymath}
\begin{aligned}
\uW^{(\tau)}&=\underset{\uW\in\GammaP}{\arg\min}\mathcal{J}_{\text{AL}}\lb{}U^{(\tau-1)},\uW,\ulam^{(\tau)}\rb\,,\\
U^{(\tau)}&=\underset{U\in\SpM}{\arg\min}\mathcal{J}_{\text{AL}}\lb{}U,\uW^{(\tau)},\ulam^{(\tau)}\rb\,,\\
\ulam^{(\tau+1)}&=\ulam^{(\tau)}+\beta\lb\uW^{(\tau)}-\ucoB\lb{}U^{(\tau)}\rb\rb\,.
\end{aligned}
\end{displaymath}
\end{Alg}

For the convergence of Algorithm~\ref{alg:original_AL_ROF}, 
where the ADMM part is approximately solved by performing only one iteration, see
~\cite{GT89}[Th. 2.2] and
~\cite{EB92}[Th. 8]. In the special case of the $TV-\ell^{2}$ problem (\ref{prob:ROF}) see~\cite{WT10} for the equivalence of minimizing (\ref{eq:variatonalproblems_L1_one}) and finding the saddle point for (\ref{eq:ROF_augmented_Lagrangian})  as well as for a detailed proof of convergence of Algorithm~\ref{alg:original_AL_ROF} in the discrete case. 

In every step of Algorithm~\ref{alg:original_AL_ROF} the two minimizers are explicitly given by (cf. Boyd and Vandebergh \cite{BV04})
\begin{eqnarray*}
\uW^{(\tau)}&=&\fuS_{\kappa}\lb{}\ucoB\lb{}U^{(\tau{}-1)}\rb-\frac{1}{\beta}\ulam^{(\tau)},\frac{1}{\beta}\rb\,,\\
U^{(\tau)}&=&\mu\lb{}\mu\coE+\beta\lb\ucoB^{\ast}\ucoB\rb\rb^{-1}(F)+\beta\lb{}\mu\coE+\beta\lb\ucoB^{\ast}\ucoB\rb\rb^{-1}\ucoB^{\ast}\lb\uW^{(\tau)}+\frac{1}{\beta}\ulam^{(\tau)}\rb\,.
\end{eqnarray*}
Choosing filters $A$ and filter families $\uBT$ determined by
\begin{eqnarray*}
\widehat{A}[k,\ell]&=&\mu\lb{}\mu+\beta\sum_{p=1}^{P}\lab\lab{}\widehat{B}_{p}[k,\ell]\rab\rab^{2}\rb^{-1}\,,\\
\widehat{\widetilde{B}}_{p}[k,\ell]&=&\beta\lb{}\mu+\beta\sum_{q=1}^{P}\lab\lab{}\widehat{B}_{q}[k,\ell]\rab\rab^{2}\rb^{-1}\widehat{B}_{p}[k,\ell]\,,
\end{eqnarray*}
for $1\leq{}p\leq{}P$, $\allkl$, verify that 
\begin{eqnarray}\label{VAR-prob-Cy:eq}
\coA = \mu\lb{}\mu\coE+\beta\ucoB^{\ast}\ucoB\rb^{-1}&\mbox{ and }&\ucoBT^{\ast}= \beta\lb{}\mu\coE
+\beta\ucoB^{\ast}\ucoB\rb^{-1}\ucoB^{\ast}\,.
\end{eqnarray}
Thus, we arrive at a special case of Problem~\ref{prob:gen_intro}.

\begin{Prob}\label{prob:intersection}
Find a point $\fix\in \Omega^{\kappa}_{1}\cap\Omega_{2}\cap \Omega_{C}$ where
\begin{equation}\label{eq:ROF_intersection}
\begin{aligned}
\Omega^{\kappa}_{1}&:=\lbb{}\lb{}U,\uW,\ulam\rb{}\in\Gamma_{1+2P}:{}\uW{}=\fuS_{\kappa}\lb\ucoB\lb{}U\rb-\frac{1}{\beta}\ulam;\,\frac{1}{\beta}\rb\rbb\,,\\
\Omega_{2}&:=\Bigg\{\lb{}U,\uW,\ulam\rb{}\in\Gamma_{1+2P}:
\\
&\quad\quad{}U=\mu\lb{}\mu\coE+\beta\ucoB^{\ast}\ucoB\rb^{-1}\lb{}F\rb+\beta\lb{}\mu\coE
+\beta\ucoB^{\ast}\ucoB\rb^{-1}\ucoB^{\ast}\lb{}\uW+\frac{1}{\beta}\ulam\rb\Bigg\}\,,\\
\Omega_{C}&:=\lbb(U,\uW,\ulam)\in\Gamma_{1+2P}:\ucoB\lb{}U\rb=\uW\rbb\,.
\end{aligned}
\end{equation}
\end{Prob}

As in the introduction, we have at once the analogue of Remark \ref{rem:equivalence}.

\begin{Rem}\label{rem:equivalence-classical}
Problem \ref{prob:intersection} is a special case of Problem \ref{prob:gen_intro}. Furthermore, every intersection point of Problem~\ref{prob:intersection} is a fixed point of Algorithm~\ref{alg:original_AL_ROF} and vice versa.
\end{Rem}

\subsection{The Generalized Hilbert Model}\label{scn:Hilbert}

Replacing the $\ell^{2}$-norm in the data-fidelity term of (\ref{prob:ROF}) with a more flexible Hilbert norm one arrives at (\ref{eq:Hilbert}), as proposed for $\ucoB=\nabla$ by Aujol and Gilboa~\cite{AG06}. It turns out that the corresponding minimization problem is also a special  case of Problem \ref{prob:gen_intro}, and \emph{strongly factoring} filters from Definition \ref{def:generalizations} (ii) assure equivalence.

\begin{The}\label{the:TV-Hilbert}
For $F\in\SpM,\uB\in\GammaP$ and $\kappa\in\lbb{}1,2\rbb$ the following hold:
\begin{itemize}\item[(i)]
If  $\Usol\in \Gamma$ is a minimizer of (\ref{eq:Hilbert}) for some $\mu\in\R_{+}$ and $M\in\SpM$ such that $\widehat{M}\in\SpM_{+}$, then
there are $\beta \in \R_{+}, A\in\SpM$ and a strong factor matrix $Y\in\SpM$ yielding $\uBT \in \GammaP$  from  $\uB$ such that $(\uB,\uBT)$ 
satisfy the~(\ref{eq:CPC}) and, with suitable $(\Wsol,\lamsol)\in\Gamma_{2P}$, $(\Usol,\Wsol,\lamsol)$ is a solution of Problem~\ref{prob:gen_intro}.

 \item[(ii)] If $(\Usol,\Wsol,\lamsol)\in\Gamma_{1+2P}$ is a solution of Problem~\ref{prob:gen_intro} with  $\kappa\in\lbb{}1,2\rbb,\beta\in\R_{+}$ and 
 $(A,\uB,\uBT)\in\Gamma_{1+2P}$ with strongly factoring $(\uB,\uBT)$ satisfying the~(\ref{eq:CPC}), then there are $\mu\in\R_{+}$ and $M\in\SpM$ with $\widehat{M}\in\SpM_{+}$ such that $\Usol$ is a minimizer of (\ref{eq:Hilbert}).
\end{itemize}

\end{The}
\begin{proof} (i): With the argument detailed in Section \ref{scn:l1-reg}, for every minimizer $\Usol\in \Gamma$ of (\ref{eq:Hilbert}) there are  $(\Wsol,\lamsol)\in\Gamma_{2P}$ such that $(\Usol,\Wsol,\lamsol)$ solves an AL problem similar to (\ref{eq:saddle}) and a corresponding saddle point equation similar to (\ref{eq:saddle}). Then \cite[Theorem 7.4, first assertion]{RTGH20} asserts the existence of a  strong factor $Y\in \Gamma$ yielding $\uBT \in \GammaP$ from  $\uB$ such that  $(\Usol,\Wsol,\lamsol)$ solves Problem~\ref{prob:gen_intro}.
 
(ii):  From Lemma \ref{lem:eigenvalues} infer that $\ucoBT^{\ast}\ucoB$ has non-negative eigenvalues, which, in case of (\ref{eq:CPC}), are then in  $[0,1)$. In consequence of \cite[Theorem 7.4, second assertion]{RTGH20} there are suitable $\mu\in\R_{+}$ and $M\in\SpM$ with $\widehat{M}\in\SpM_{+}$ such that $\Usol$ is a minimizer of (\ref{eq:Hilbert}).
\end{proof}

\begin{Rem}\label{rem:already}
\begin{enumerate}[(i)]

\item The functional $\mathcal{J}_{\text{HIL}}$ of (\ref{eq:Hilbert}) is strictly convex, coercive and lower semi-continuous, yielding the existence of a unique minimizer. Moreover, its AL/ADMM-algorithm is Algorithm~\ref{alg:Generalized_Algorithm} which convergences by virtue of Eckstein and Bertsekas~\cite{EB92}[Theorem 8].

\item Theorem~\ref{the:TV-Hilbert} can be extended to feature strongly factoring $(\uB,\uBT)$ that only satisfy the~(\ref{eq:NEPC}) when allowing the entries of $\widehat{M}$ to be in $\R_{+}\cup\lbb{}0\rbb{}$. In this case a minimizer is not necessary unique, since $\fuJ_{\text{HIL}}$ is not necessary strongly convex. 

\end{enumerate}
\end{Rem}

\section{Proof of the Main Theorem}\label{sec:four}

For $\kappa = 1,2$ consider the set valued functions
\begin{equation}\label{eq:fuY_A}
\begin{aligned}
\fuY^{\kappa}:\SpM&\rightrightarrows\GammaP\,,\\
U&\mapsto{}\lbb{}\ulam\in\GammaP: 0 \in \partial|_{\uW = \ucoB(U)}\left(  \lab\uW\rab_{1,\kappa}+\lla\ulam,\uW\rra\right)\rbb\,.
\end{aligned}
\end{equation}
Here, $\partial|_{\uW = \ucoB(U)}$ denotes the set valued subdifferential (e.g. \cite{BC11}) of a convex function with respect to $\uW$ evaluated at $\ucoB(U)$. Defining
$\fuSh^{\kappa} := {\rm graph}(\fuY^{\kappa})$
we have,
\begin{equation}\label{eq:gen_ulam_A}
\begin{aligned}
&\fuSh^{1}=\\
&\,\,\lbb{}\lb{}U,\ulam\rb\in\Gamma_{1+P}:
\lambda_{p}[k,\ell]\begin{cases}
=-1 & \text{if }\coBp\lb{}U\rb[k,\ell]>0\\
\in[-1,1] & \text{if }\coBp\lb{}U\rb[k,\ell]=0\\
=1 & \text{if }\coBp\lb{}U\rb[k,\ell]<0
\end{cases}\,,\text{ for all }k,\ell,p\rbb\,,
\end{aligned}
\end{equation}
and 
\begin{equation}\label{eq:gen_ulam_B}
\begin{split}
&\fuSh^{2}=\Bigg\{\lb{}U,\ulam\rb\in\Gamma_{1+P}:\\
&(\lambda_{p}[k,\ell])_{p=1}^{P}\begin{cases}
=-\frac{(\coBp(U)[k,\ell])_{p=1}^{P}}{\lab\lab{}(\coBp(U)[k,\ell])_{p=1}^{P}\rab\rab} &\text{ if }\lab\lab{}(\coBp(U)[k,\ell])_{p=1}^{P}\rab\rab>0\\
\in\mathcal{K}_{1}(0) &\text{ if }\lab\lab(\coBp(U)[k,\ell])_{p=1}^{P}\rab\rab=0
\end{cases}\,,\text{ for all }k,\ell\Bigg\}\,,
\end{split}
\end{equation}
where $\mathcal{K}_{1}(0)$ is the closed unit ball around the origin. We also need the following set, 
\begin{displaymath}
\Omega_F^G:=\lbb{}(U,\ulam)\in\Gamma_{1+P}:U=\lb\coE-\ucoBT^{\ast}\ucoB\rb^{-1}\coA(F)+\frac{1}{\beta}\lb{}\coE-\ucoBT^{\ast}\ucoB\rb^{-1}\ucoBT^{\ast}\lb\ulam\rb\rbb\,,
\end{displaymath}
which is well defined in case the~(\ref{eq:CPC}) holds as proof of Lemma~\ref{lem:5_and_6} teaches.
In order to prove Theorem \ref{the:existenceouter} we need to show that $\Omega^{\kappa}_{1}\cap \Omega_{2}^{G}\cap \Omega_{C} \neq\emptyset $. The following Lemma asserts that we can  equivalently show that $\fuIm^{G}\cap \fuSh^{\kappa} \neq\emptyset$.

\begin{Lem}\label{lem:5_and_6}
Let $F\in\SpM,\beta\in\R_{+}$ and $(A,\uB,\uBT)\in\Gamma_{1+2P}$ with weakly factoring $(\uB,\uBT)$ satisfying the (CPC). Then $(U,\ulam) \in  \fuIm^{G}\cap \fuSh^{\kappa}$ if and only if there exists $\uW\in\Gamma_{P}$ such that $(U,\uW,\ulam) \in\Omega^{\kappa}_{1}\cap \Omega_{2}^{G}\cap \Omega_{C} $.
\end{Lem}

\begin{proof}
As shown in Lemma \ref{lem:eigenvalues}, if $(\uB,\uBT)$ factors weakly, then $\ucoBT^\ast\ucoB$ diagonalizes and in conjunction with (CPC), all eigenvalues are non-negative and strictly less than one. In consequence, $\coE-\ucoBT^{\ast}\ucoB$ is invertible and hence
\begin{displaymath}
\begin{aligned}
U&=\lb\coE-\ucoBT^{\ast}\ucoB\rb^{-1}\coA(F)+\frac{1}{\beta}\lb{}\coE-\ucoBT^{\ast}\ucoB\rb^{-1}\ucoBT^{\ast}\lb\ulam\rb
\end{aligned}
\end{displaymath}
is equivalent with
\begin{displaymath}
\begin{aligned}
U-\ucoBT^{\ast}(\uW)&=\coA(F)+\frac{1}{\beta}\ucoBT^{\ast}(\ulam)
\end{aligned}
\end{displaymath}
yielding
$$ (U,W,\ulam)\in\Omega_{2}^{G}\cap  \Omega_{C} ~ \Leftrightarrow ~(U,\ulam) \in  \fuIm^{G}\,.$$

Further, by definition we have at once that $(U,\ulam) \in \fuSh^{\kappa}$ implies that $(U,W,\ulam) \in \Omega^{\kappa}_{1}$ for all $\uW\in \GammaP$. Vice versa, if $\uW = \ucoB(U)$, then verify that the inverse statement, $(U,W,\ulam) \in \Omega^{\kappa}_{1}$ implies that $(U,\ulam) \in \fuSh^{\kappa}$, is also true.
\end{proof}

Next, we show that
\begin{equation}\label{eq:fuShnull}
\fuSh^{0}:=\lbb{}\lb{}U,\ulam\rb\in\Gamma_{1+P}:\ulam=\underline{0}\rbb\,.
\end{equation}
is an affine complement of the affine space $\fuIm^G\subset \Gamma_{1+P}$. 

\begin{Lem}\label{lem:fuShandfuIm}
Let $F\in\SpM,\beta\in\R_{+}$ and $(A,\uB,\uBT)\in\Gamma_{1+2P}$ 
with weakly factoring $(\uB,\uBT)$ satisfying the~(\ref{eq:CPC}). Then
$$ \fuSh^{0} \oplus \fuIm^{G} = \Gamma_{1+P}\,.$$ 
\end{Lem}

\begin{proof} As argued in the proof of Lemma \ref{lem:5_and_6}, weakly factoring filters satisfying the~(\ref{eq:CPC}) imply that $\coE-\ucoBT^{\ast}\ucoB$ is invertible and thus
	\begin{displaymath}
	\lb\lb\coE-\ucoBT^{\ast}\ucoB\rb^{-1}\coA(F),\underline{0}\rb\in\fuSh^{0}\cap\fuIm^{G}\,.
	\end{displaymath}
	To see uniqueness suppose $\lb{}U_{1},\ulam_{1}\rb,\lb{}U_{2},\ulam_{2}\rb\in\fuSh^{0}\cap\fuIm^{G}$. Then $\ulam_1 = 0 = \ulam_2$ 
	yielding,
	\begin{displaymath}
	U_{1}-U_{2}=\frac{1}{\beta}\lb{}\coE-\ucoBT^{\ast}\ucoB\rb^{-1}\ucoBT^{\ast}\lb\ulam_1-\ulam_2\rb= 0 \,. 
	\end{displaymath}
	The dimension of $\fuSh^{0}$ is $nm$, while that of $\fuIm^{G}$ is $Pnm$ as being the graph of a linear function from $\GammaP$ to $\SpM$. Since we just showed $\text{dim}(\fuSh^{0}\cap\fuIm^{G})=0$, we have indeed $\fuSh^{0}\oplus\fuIm^{G}=\Gamma_{1+P}$.
\end{proof}

In the following, let $U_0 \in \SpM$ give rise to the unique intersection point 
\begin{equation}\label{eq:unique-affine-intersection}(U_0,\underline{0}) \in\fuSh^{0}\cap\fuIm^{G}
\end{equation}
 guaranteed by Lemma \ref{lem:fuShandfuIm}. 
 
 Further let $\theta \in (0,\pi/2]$ denote the \emph{first principal angle} between $\fuSh^{0}\cap\fuIm^{G}$. In general for two affine subspaces $\Xi,\Omega \subset \R^D$, each of positive dimension, with unique intersection point $v\in \Xi\cap\Omega$, the first principle angle between $\Xi$ and $\Omega$ is defined by
$$\underset{\begin{array}{c}\xi+v\in{}\Xi:\lab\lab{}\xi\rab\rab=1\\
\omega+v\in{}\Omega:\lab\lab{}\omega\rab\rab=1\end{array}}{\max}\lb{}\xi^{T}\omega\rb^{2}\\
=\cos^{2}(\theta) \in [0,1)\,,
$$
cf. \cite{Jor75}. We need the following estimate in the additional case of $\R^D = \Xi \oplus \Omega$, the proof of which is deferred to the appendix. 

\begin{Lem}\label{lem:intersection}
With the above assumptions and notations suppose that $\phi : \Omega \to \R^D$ is continuous, satisfying \begin{displaymath}
\lab\lab{}\phi(\omega)-\omega\rab\rab\leq{}C\,,\text{ for all }\omega\in{}\Omega\,,
\end{displaymath}
for a constant $C>0$. Then $\Xi\cap{}\phi(\Omega)\neq \emptyset$ and 
\begin{displaymath}
\lab\lab{}u-v\rab\rab\leq{}C\lb{}1+\frac{1}{\sin(\theta)}\rb\,,
\end{displaymath}
for all $u\in{}\Xi\cap\phi(\Omega)$.
\end{Lem}


In application of Lemma \ref{lem:intersection}, set $C := Pnm+1$ and $R:=C(1+\nicefrac{1}{\sin(\theta)})$ to obtain the compact ball
\begin{displaymath}
\mathcal{K}:=\lbb{}U\in\SpM:\lab\lab{}U-U_{0}\rab\rab\leq{}R\rbb\,,
\end{displaymath}
around $U_{0}$. 

In the next step, we approximate $\fuY^{\kappa}$ by a single-valued function $v_r^\kappa$, $r>0$,  in order to consider $r\to 0$ afterwards. We use the \emph{separation of $P$ from $Q$} 
$$d(P,Q) := 
\sup_{p\in P}\lb\inf_{q\in Q}\lb\lab\lab{}p-q\rab\rab\rb\rb\,,$$
for $P,Q\subset \R^D$, cf. \cite{Cel69}.

\begin{Lem}\label{lem:approxcont}
 For every $r>0$ there is a continuous mapping $v_r^\kappa: \mathcal{K} \to \Gamma_{1+P}$ such that $d\left({\rm graph}(v_r^\kappa),{\rm graph}(\fuY^\kappa|_{\mathcal{K}})\right)\leq  r$. 
\end{Lem}

\begin{proof}
 By definition, cf. (\ref{eq:gen_ulam_A}) and (\ref{eq:gen_ulam_B}), the functions $\fuY^{\kappa}$ are convex valued (i.e. $\fuY^{\kappa}(U)$ is a convex set for every $U \in \SpM$) and their graphs $\fuSh^{\kappa}$ are closed (due to continuity of $\ucoB$). Since they map to convex subsets of a compact set, they are upper semi-continuous (e.g. \cite{Cel69}[p. 19]). Thus  \cite{Cel69}[Theorem 1] is applicable, yielding the existence of  $v_r^\kappa$.
\end{proof}

With $v_r:=v_r^\kappa$ (for ease of notation dropping the superscript $\kappa$) from Lemma \ref{lem:approxcont}, we have thus that for every $U\in \mathcal{K}$ there are $U' \in \mathcal{K}$ and $\ulam' \in \fuY^\kappa(U')$ such that 
$$ \|U-U'\|^2 + \|v_r(U) - \ulam'\|^2 \leq r^2\,.$$
Since $\|\ulam'\|^2 \leq Pnm$ by definition of $\fuY^\kappa$ in ( \ref{eq:gen_ulam_A}) and (\ref{eq:gen_ulam_B}), we have
$$ \|v_r(U) \| \leq Pnm + 1\mbox{ for all } U \in \mathcal{K}\mbox{ and }r<1\,. $$
In consequence, the single-valued, continuous mapping
\begin{displaymath}
\phi_{r}:\fuSh^{0}\to\Gamma_{1+P},\quad (U,\underline{0})\mapsto (U,\upsilon_{r}(\mathcal{P}_{\mathcal{K}}(U)))\,,
\end{displaymath}
where $\mathcal{P}_{\mathcal{K}}$ is the orthogonal projection of $\SpM$ onto $\mathcal{K}$, satisfies
$$ \|\phi_r(\omega) - \omega\| \leq Pmn +1=C\,.$$
Since $(U_0,\underline{0})$ is the unique intersection point of $\fuSh^{0}$ and $\fuIm^G$, cf. (\ref{eq:unique-affine-intersection}) and Lemma \ref{lem:fuShandfuIm}, application of Lemma \ref{lem:intersection} yields for all $0< r< 1$ the existence of
$ (U_r,\ulam_r) \in \fuIm^{G}\cap\phi^{\kappa}_{r}\lb{}\fuSh^{0}\rb$ with the property 
$$\|(U_r,\ulam_r) - (U_0,\underline{0})\| \leq (Pmn+1) \left(1 + \frac{1}{\sin \theta}\right)\,.$$
We conclude that $U_r\in \mathcal{K}$ and hence $\ulam_r = v_r(U_r)$. Moreover, since the sequence $(U_r,\ulam_r) \in {\rm graph}(v_r)$ ($r= 1/n, n\in \N$) is bounded, it has a cluster point  
$(U^\dagger,\ulam^\dagger) \in \fuIm^G$ because $\fuIm^G$ is closed. Further, since $d({\rm graph}(v_r),{\rm graph}(\fuY^\kappa))\to 0$ with respect to the separation $d$, 
$(U^\dagger,\ulam^\dagger)$ is in the closure of ${\rm graph}(\fuY^\kappa) = \fuSh^\kappa$, which is closed, as remarked above in the proof of Lemma \ref{lem:approxcont}. Hence
$$ (U^\dagger,\ulam^\dagger) \in  \fuIm^G \cap \fuSh^{\kappa}\,,$$
yielding the assertion of Theorem \ref{the:existenceouter}, due to Lemma \ref{lem:5_and_6}.

\section{Numerical Experiments}\label{sec:five}

In illustration of the potential of our generalized algorithm we consider two classic problems in image decomposition: denoising and cartoon-texture decomposition. As proof of concept 
we show that  ad-hoc filters yield results comparable or even better than classical filters. Subjecting such filters to methods learning over larger classes is beyond the scope of this work and left for future research.

\subsection{Building New Filters for Denoising}\label{sec:denoising}

We illustrate how new (weakly or strongly) factoring filters can be easily obtained from existing models, that, as they feature more parameters, outperform existing models. To this end, we consider four classic test images $U\in\SpM$ (taken from~\cite{gra}) listed in Table \ref{tab:toy} with additive Gaussian noise $\epsilon$, 
$$F=U+\varepsilon\,,\text{ where }
\varepsilon\sim\mathcal{N}(0,\coAx\coAx^{\ast})\,,$$
correlated in horizontal direction,
\begin{equation*} 
Z_{x}=\frac{\sqrt{50}}{20}\begin{pmatrix}
    4 & 4 & 2 & 1 & 0 & \cdots{} & 0 & 1 & 2 & 4 \\
    1 & 0 & 0 & 0 & 0 &          &   & 0 & 0 & 0 \\
    0 & 0 & \ddots{}  & & &      &   &   & 0 & 0 \\
    \vdots{} &&&&&&&&& \vdots{} \\
    0 & 0 & &&&&&&\ddots{} & 0 \\
    1 & 0 & \cdots{} &&&&&\cdots{}& 0 & 0 	
	\end{pmatrix}\,.		
\end{equation*}
We compare three models for denoising and for each image we perform 100 experiments and record their respective peak signal to noise ratio (PSNR) mean $\pm$ standard deviation in Table \ref{tab:toy}. In all 100 experiments the PSNR of Models II and III outperformed the PSNR of Model I.  

\paragraph{Model I: classical $TV-\ell^{2}$.} As illustrated in (\ref{VAR-prob-Cy:eq}), denoting by  $\uD \in \Gamma^2$ the filters yielding the discrete derivative operator, i.e. $\nabla=\ucoD$, in this model we have
\begin{eqnarray}\label{VAR-prob-Cy-2:eq}
\ucoB = \ucoD\mbox{ and }\ucoBT^{\ast} = \coY \ucoB^\ast&\mbox{ where }&\coY = \beta\lb{}\mu\coE
+\beta\ucoB^{\ast}\ucoB\rb^{-1}\,,
\end{eqnarray}
as well as $\coA = \mu \lb{}\mu\coE +\beta\ucoB^{\ast}\ucoB\rb^{-1}$ where $\mu>0$ is the one weight optimized over.

\paragraph{Model II: strongly factoring with spectrally weighted gradient coordinates.}
Inspired from (\ref{VAR-prob-Cy-2:eq}) we place different weights on the spectral coordinates of the derivative operator via
$$ \ucoB = \ucoD \coV\mbox{ and }\coY =\beta\lb{}\mu\coE
+\beta\ucoB^{\ast}\ucoB\rb^{-1}\coVT^2 $$
where
\begin{displaymath}
\widehat{V}[k,\ell]:=\exp\lb{}-y_{1}\omega_{k}^2-y_{2}\omega_{\ell}^2\rb = \widehat{\widetilde{V}}[k,\ell]^{-1}\,,\quad\allkl\,
\end{displaymath}
and
\begin{equation}\label{omega-def:eq}
\omega_{k}:=\begin{cases}\frac{2\pi{}k}{n} & \text{ if }k<\frac{n}{2}\\
-2\pi+\frac{2\pi{}k}{n} & \text{else}\end{cases}\quad\text{ and }\quad\omega_{\ell}:=\begin{cases}\frac{2\pi{}\ell}{m} & \text{ if }\ell<\frac{m}{2}\\
-2\pi+\frac{2\pi\ell}{m} & \text{else}\end{cases}\,,
\end{equation} 
such that $\ucoBT^\ast = \coY \ucoB^\ast =  \beta\lb{}\mu\coE +\beta\ucoD^{\ast}\ucoD\rb^{-1}\mathfrak{C}_{\widetilde{V}}\ucoD^{\ast}$, cf. Remark \ref{ast-factor:eq}.
With $\coA$ unchanged, by definition, we arrive at a strongly factoring model with three weights $\mu, y_1,y_2$ optimized over. 

\paragraph{Model III: weakly factoring with spectrally weighted gradient directions.} Now, taking $Y^{TV-\ell^2}$ from the classical $TV-\ell^2$-model above in (\ref{VAR-prob-Cy-2:eq}) with $\coYROF=\beta\lb{}\mu\coE
+\beta\ucoB^{\ast}\ucoB\rb^{-1}$ and defining
$$\ucoYROF :\GammaP \to \GammaP, \uH \mapsto \big(\coYROF (H_1),\ldots, \coYROF (H_P)\big)\,,$$ 
we place different weights on the spectral representation of the directional derivatives via
$$ \ucoB = \ucoR\,\ucoD\mbox{ and }\ucoYY = \ucoYROF  \ucoRT^2 $$
where $\uR = (R_1,R_2)\in \Gamma^2$ with
\begin{displaymath}
\widehat{\widetilde{R}}_{1}[k,\ell]:=r_{1}=:\widehat{R}_{1}[k,\ell]^{-1}\mbox{ and }\widehat{\widetilde{R}}_{2}[k,\ell]:=r_{2}=:\widehat{R}_{2}[k,\ell]^{-1}\,, r_1,r_2>0\,,
\end{displaymath}
for all $\allkl$,
%
such that 
$$\ucoBT^\ast = \ucoB^\ast \ucoYY = \beta\lb{}\mu+\beta\ucoD^{\ast}\ucoD\rb^{-1}\ucoD^{\ast} \ucoRT\,,$$ 
cf. Remark \ref{ast-factor:eq}.
With $\coA$ unchanged, by definition, we arrive at a weakly factoring model with three weights $\mu, r_1,r_2$ optimized over. Whenever $r_1\neq r_2$ this model is not strongly factoring.
 
 \begin{table}[!t]
\captionsetup{justification=centering, labelsep=newline}
\caption{\it Trained parameters; $\mu$ has not been trained for Models II and III but the optimal values from Model I have been used.}\label{tab:parameters}
\centering
\begin{tabular}{c|c|c|c|c|c} 
Image &  Model I  & \multicolumn{2}{|c|}{Model II} & \multicolumn{2}{|c}{Model III}\\
& $\mu$ & $y_{1}$ & $y_{2}$  & $r_{1}$ & $r_{2}$ \\      
\hline
 &&&&&\\[-0.8em]
\multirow{1}{*}{barbara}
& 0.0818 &  -0.07  & 0.042  & 0.7408  & 1.3499  \\
&&&&&\\[-0.8em]
\multirow{1}{*}{cameraman}
& 0.054 &  -0.07  & 0.014  & 0.7408  & 1.3499  \\
 &&&&&\\[-0.8em]
\multirow{1}{*}{boat}
 & 0.051 &  -0.07  & 0.014  & 0.7408  & 1.3499 \\
 &&&&&\\[-0.8em]
 \multirow{1}{*}{goldhill}
 & 0.0458 &  -0.07  & 0.014  & 0.7408  & 1.1972 \\
\end{tabular}
\end{table}

\begin{table}[!h]
\captionsetup{justification=centering, labelsep=newline}
\centering
\caption{\it Peak signal to noise ratio (PSNR): mean $\pm$ standard deviation from one hundred experiments each with \underline{underlined} best mean PSNR  and  \textbf{bold} mean PSNR exceeding the mean PSNR of Model I ($TV-\ell^{2}$) by more that $0.1$.  
}\label{tab:toy}
\resizebox{\columnwidth}{!}{
\begin{tabular}{c|c|c|c} 
Image & \multicolumn{3}{|c}{PSNR} \\
&  Model I & Model II & Model III
\\
\hline
 &&&\\[-0.8em]
\multirow{1}{*}{barbara}
&  $25.1158\pm{}0.0251$  & \textbf{25.4197\textpm{}0.0266}  & \underline{\textbf{25.5282\textpm{}0.0257}} \\
&&&\\[-0.8em]
\multirow{1}{*}{cameraman}
&  $26.8547\pm{}0.0664$  & \textbf{27.0461\textpm{}0.0582}  & \underline{\textbf{27.0957\textpm{}0.0689}} \\
 &&&\\[-0.8em]
\multirow{1}{*}{boat}
 & $27.0946\pm{}0.0296$  & \textbf{27.2183\textpm{}0.0273}  & \underline{\textbf{27.3681\textpm{}0.0303}}  \\
 &&&\\[-0.8em]
 \multirow{1}{*}{goldhill}
 &  $27.4352\pm{}0.0322$  & $27.4956\pm{}0.0326$  & \underline{$27.5305\pm{}0.0328$}  \\
\end{tabular}}
\end{table}

\paragraph{Model training.}
The model unspecific parameters are set to $\beta=0.1$ and $\kappa=2$, while we compute the solution of Problem~\ref{prob:gen_intro} with  Algorithm~\ref{alg:Generalized_Algorithm} up to accuracy $\epsilon = 10^{-5}$.
For Model I, this accuracy was reached well before $500$ iterations, for some of the runs for Models II and III after $500$ iterations only an accuracy of $\epsilon = 10^{-4}$ has been reached. While $\mu$ has been optimized via bisection, the parameters $r_1,r_2$ and $y_1,y_2$ have been optimized via a grid search, cf. \cite{Ric19b}.

\subsection{A Laplace-Spline-Riesz Model Separating Cartoon and Texture}\label{sec:cartoon}

As another illustration of our approach, we now build an elaborate filter family separating cartoon and texture. Consider the  \emph{isotropic polyharmonic B-splines} which, for  $\gamma \geq 1/2$, are given in the frequency domain by
\begin{equation}\label{eq:Spline}
\widehat{f}_{\gamma}(x,y):=\lb\frac{4\lb{}\sin^{2}\lb\frac{x}{2}\rb+\sin^{2}\lb\frac{y}{2}\rb\rb-\frac{8}{3}\lb{}\sin\lb\frac{x}{2}\rb\sin\lb\frac{y}{2}\rb\rb}{\lb{}x^{2}+y^{2}\rb}\rb^{\frac{\gamma}{2}}\,.
\end{equation} 
Van de Ville et al.~\cite{VBU05} explain how these splines are basis-splines (B-splines) for the discrete fractional Laplace operator $(\ucoD^*\ucoD)^\gamma$, and how, using (\ref{eq:Spline}) as the scaling function, a bi-orthogonal wavelet basis is constructed, drawing on a similar relation as between the Haar wavelet frame and the discrete gradient (see for example Cai et al.~\cite{CDOS12}). Using dyadic subsampling in the construction of the wavelet frame operators yields $3$ wavelets per scale. We consider in this example $J+1=4$ scales to derive the high-pass filter pair $(\uB,\uBT)$. Additionally, we let $(\uB,\uBT)$ involve Riesz transforms with $Z= 3$ directions, cf. \cite{USV09,UV10}, giving in total families with $P = 3Z(J+1)=36$ members. All of this is detailed in Appendix \ref{sep-cartoon-texture:appendix}.

With the 
 \emph{autocorrelation function}
\begin{equation}\label{eq:Auto}
a_{\gamma}(x,y):=\sum_{r,s\in\Z}\lb{}f_{\gamma}(x+2\pi{}r,y+2\pi{}s\rb^{2}\,,
\end{equation}
we obtain a low-pass filter $A$ determined by 
\begin{displaymath}
\widehat{A}[k,\ell]:=\lab\lab\frac{f_{\gamma}(2^{J}\omega_{k},2^{J}\omega_{\ell})}{\sqrt{a_{\gamma}(2^{J}\omega_{k},2^{J}\omega_{\ell})}}\rab\rab^{2}\,,
\end{displaymath}
with $\omega_{k}, \omega_{\ell}$ taken from (\ref{omega-def:eq}). As detailed in Appendix \ref{sep-cartoon-texture:appendix}, the resulting model is weakly factoring, it does not satisfy the (NEPC), however (some eigenvalues of $\ucoBT^\ast \ucoB$ exceed 1). To this end, introduce ($0\leq j \leq J,~1\leq z \leq Z,~1\leq s \leq 3$)
\begin{displaymath}
H[k,\ell]=\lb\widehat{A}[k,\ell]+\sum_{j,z,s}\overline{\widetilde{B}_{j,z,s}[k,\ell]}\widehat{B}_{j,z,s}[k,\ell]\rb^{-\frac{1}{2}}\,,
\end{displaymath}
to obtain the corrected matrix families 
\begin{equation}\label{adjusted-filters:eq}
\widehat{A}^{\text{cor}}:=H\odot{}H\odot{}\widehat{A}\,,~\widehat{B}^{\text{cor}}_{(j,z,s)}:=H\odot{}\widehat{B}_{(j,z,s)}\,,~\widehat{\widetilde{B}}^{\text{cor}}_{(j,z,s)}:=H\odot{}\widehat{\widetilde{B}}_{(j,z,s)}\,.
\end{equation}
Then, the weakly factoring model resulting (with an additional minor adjustment leading to $\uB^{\text{adj}},\uBT^{\text{adj}}$, detailed in Appendix \ref{sep-cartoon-texture:appendix}) from $(A^{\text{cor}},\uB^{\text{adj}},\uBT^{\text{adj}})$ satisfies (numerically verified) the (CPC) and we call it the \emph{Laplace-spline-Riesz} (LsR) model. We compare it to  the $TV-\ell^{2}$-model, cf. (\ref{prob:ROF}),the $OSV$-model (\cite{OSV03}, as an approximation of the  $TV-G$-model) and the $TV$-Hilbert model ((\ref{eq:Hilbert}) with $\uB = \uD$ and $M$ as detailed in \cite{BLMV10}[Eq.8].).
%

Since there is no benchmarking available for cartoon-texture decomposition, in Figures~\ref{fig:c_t_1} and~\ref{fig:c_t_2} we inspect desirable features of resulting cartoons from five classic images from~\cite{gra}, zooming in on specific details; for the complete set of cartoons obtained see Figures~\ref{fig:c_t_1_full} and~\ref{fig:c_t_2_full} in Appendix~\ref{app:cartoon}. We have manually chosen parameters (see Table \ref{texture-contrast:fig}) for each model, respectively, in order to remove all texture (specifically from the table cloth) from the "barbara" image (first column of Figure  \ref{fig:c_t_1}) while keeping all "large-scale" structures (specifically the upper edge separating the table cloth from the floor behind).

\begin{table}[!h]
\centering
 \begin{tabular}{c|ccccc}Model &$\kappa$ &$\mu$&$\beta$& $\nicefrac{1}{(2\pi\sigma)^4}$ & $\gamma$\\ \hline
$TV-\ell^2$ & $2$ &$0.02$&$1$& - & -\\
$OSV$      & $2$ &$0.01$&$1$& - & -\\
$TV-$Hilbert& $2$ &$0.125$&$1$& $0.125$ & -\\
LsR        & $1$ & - & $1$ & - & $1.2$\\
 \end{tabular}
 \caption{\it Parameters to remove texture while keeping contrast (for the definition of $\sigma$ see~\cite{BLMV10}[Eq.8]). While in the first three models the choice of $\beta$ has no effect on the solution, in the LsR model it does. \label{texture-contrast:fig}}
\end{table}

While $TV-\ell^2$ and $OSV$ retain some texture, especially along edges, (Figures~\ref{fig:c_t_1} (d) and (h) as well as Figures~\ref{fig:c_t_2} (d) and (h)), edges are sharper for $TV-$Hilbert and LsR (Figures~\ref{fig:c_t_1} (l) and (p) as well as Figures~\ref{fig:c_t_2} (j) and (m)). 
The same effects are visible in Column 3 of Figure~\ref{fig:c_t_1}. For  LsR this comes at a price of slight blurring and light artefacts between the dark squares (Figure~\ref{fig:c_t_1}, (r)). On the other hand, the irregular water pattern (Column 4  of Figure~\ref{fig:c_t_1}) is only fully removed by  LsR (Figure~\ref{fig:c_t_1}, (s)). Notably, only  $TV-\ell^2$ was not able to remove all texture from the table cloth without loosing the contrast between the upper edge of the table cloth and the floor behind (Figure ~\ref{fig:c_t_1}, (d)). Similarly only  $TV-\ell^2$ looses the contrast between the middle and bottom part of the sail (Figure~\ref{fig:c_t_2}, (e)).
Finer scale texture patterns on the fish (Column 3 of Figure~\ref{fig:c_t_2}) are again fully removed only by LsR ((Figure~\ref{fig:c_t_1}, (r)).

\begin{figure}
\centering
\vspace{-2cm}
\subfloat[][]{
	\includegraphics[scale=0.21]{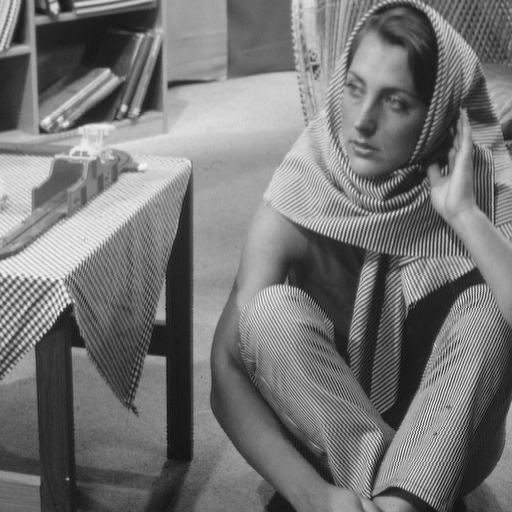}
}
\subfloat[][]{
	\includegraphics[scale=0.21]{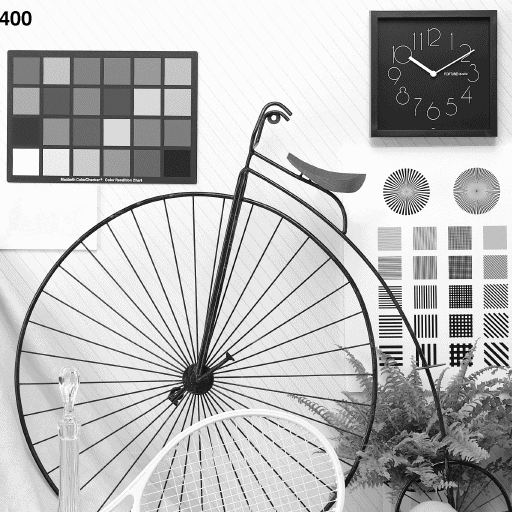}
}
\subfloat[][]{
	\includegraphics[scale=0.21]{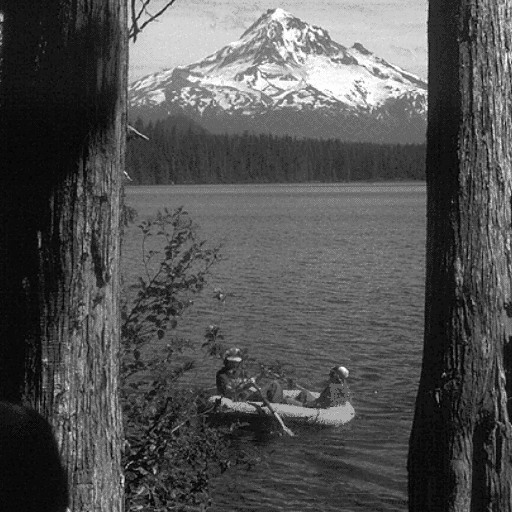}
}
\\
\subfloat[][]{
	\includegraphics[scale=0.55]{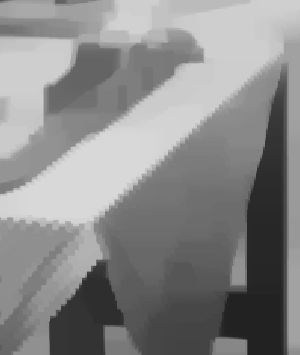}
}
\subfloat[][]{
	\includegraphics[scale=0.37]{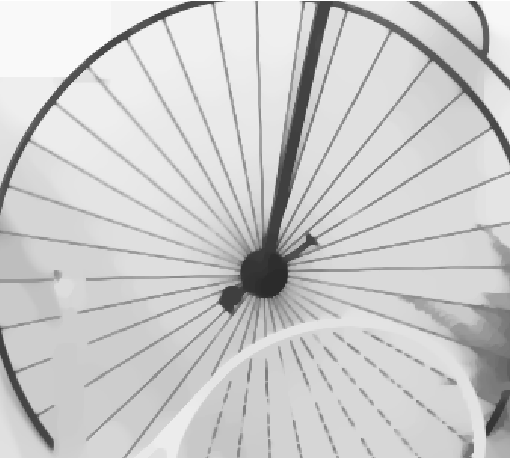}
}
\subfloat[][]{
	\includegraphics[scale=0.65]{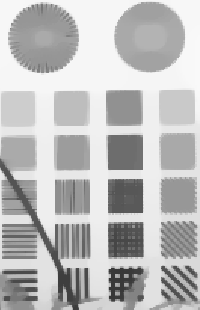}
}
\subfloat[][]{
	\includegraphics[scale=0.75]{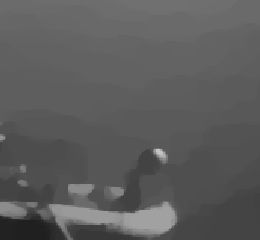}
}
\\
\subfloat[][]{
	\includegraphics[scale=0.55]{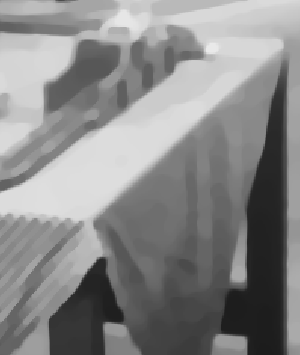}
}
\subfloat[][]{
	\includegraphics[scale=0.37]{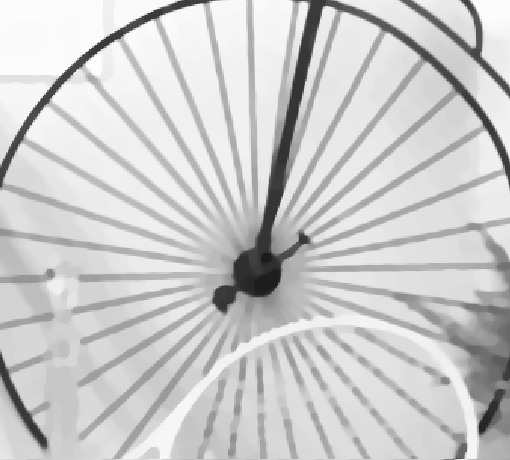}
}
\subfloat[][]{
	\includegraphics[scale=0.65]{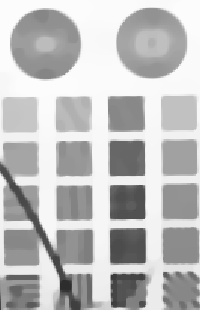}
}
\subfloat[][]{
	\includegraphics[scale=0.75]{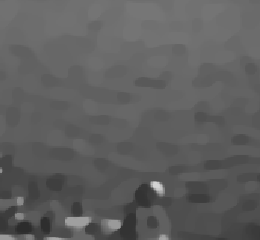}
}
\\
\subfloat[][]{
	\includegraphics[scale=0.55]{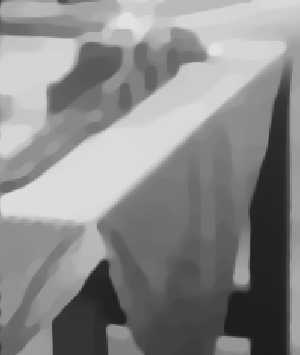}
}
\subfloat[][]{
	\includegraphics[scale=0.37]{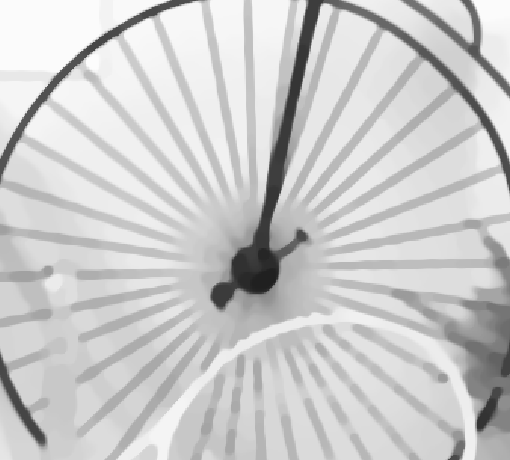}
}
\subfloat[][]{
	\includegraphics[scale=0.65]{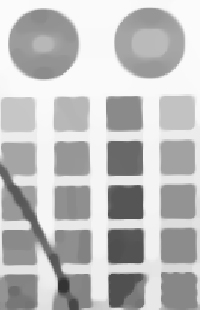}
}
\subfloat[][]{
	\includegraphics[scale=0.75]{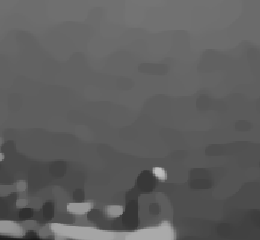}
}
\\
\subfloat[][]{
	\includegraphics[scale=0.55]{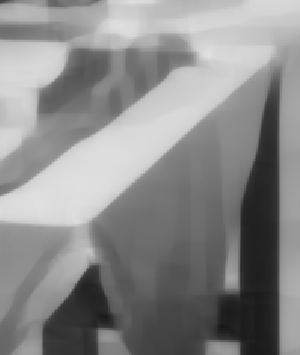}
}
\subfloat[][]{
	\includegraphics[scale=0.37]{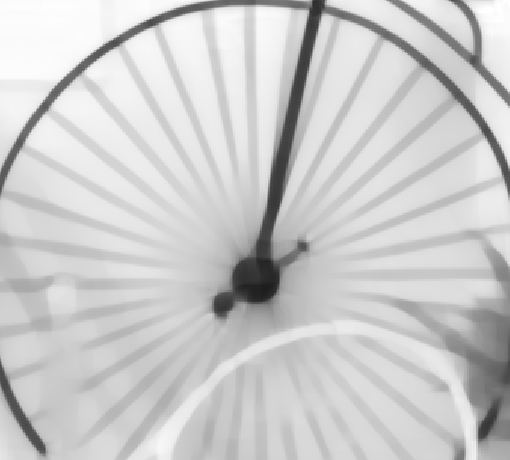}
}
\subfloat[][]{
	\includegraphics[scale=0.65]{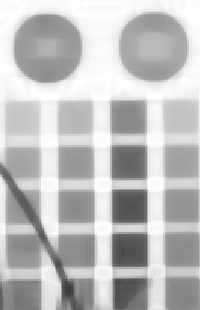}
}
\subfloat[][]{
	\includegraphics[scale=0.75]{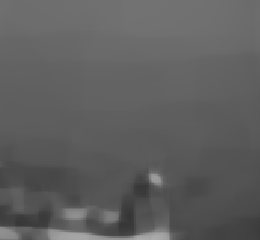}
}
\caption{\it Test images from \cite{gra} (1st row): barbara (a), bicycle (b), lake (c). 
Enlarged details from cartoons computed via $TV-\ell^{2}$ (2nd row), $OSV$ (3rd row), $TV$-Hilbert (4th row) and our LsR-model (5th row).} 
\label{fig:c_t_1}
\end{figure}

\begin{figure}
\centering
\vspace{-2cm}
\subfloat[][]{
	\includegraphics[scale=0.21]{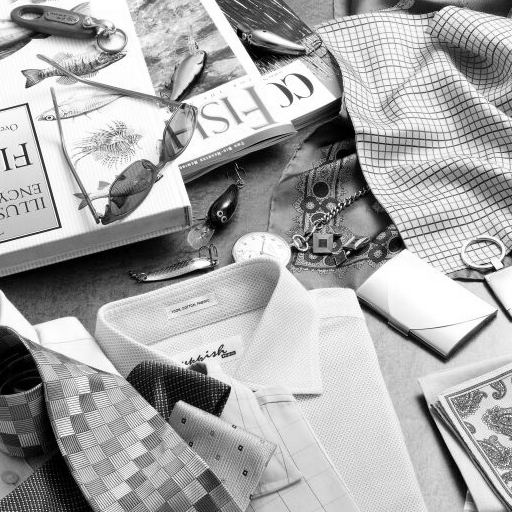}
}
\subfloat[][]{
	\includegraphics[scale=0.21]{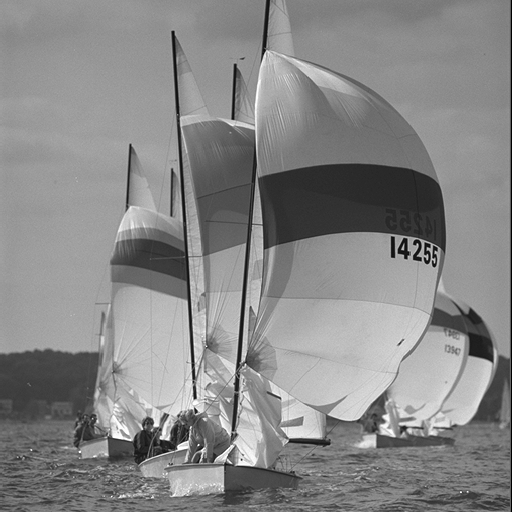}
}
\subfloat[][]{
	\includegraphics[scale=0.21]{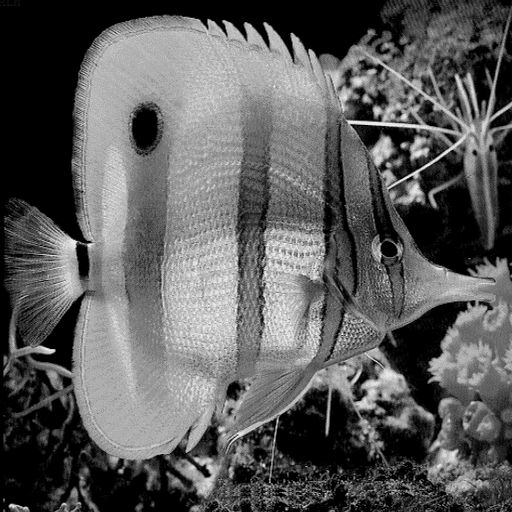}
}
\\
\subfloat[][]{
	\includegraphics[scale=0.47]{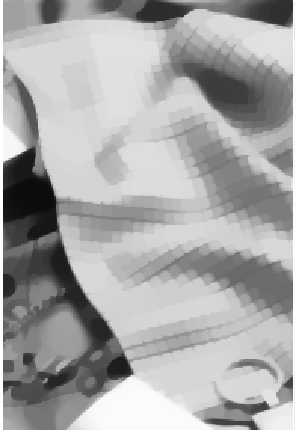}
}
\subfloat[][]{
	\includegraphics[scale=0.78]{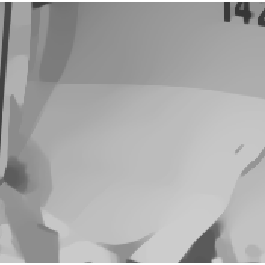}
}
\subfloat[][]{
	\includegraphics[scale=0.67]{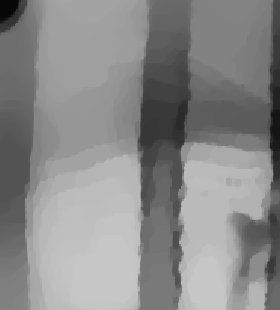}
}
\\
\subfloat[][]{
	\includegraphics[scale=0.47]{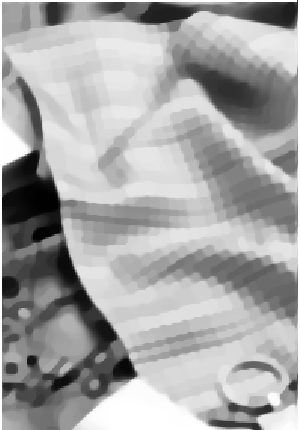}
}
\subfloat[][]{
	\includegraphics[scale=0.78]{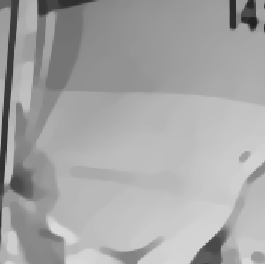}
}
\subfloat[][]{
	\includegraphics[scale=0.67]{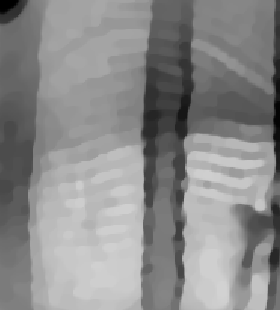}
}
\\
\subfloat[][]{
	\includegraphics[scale=0.47]{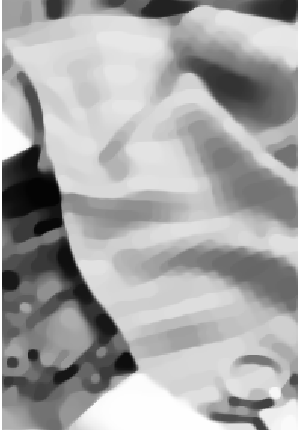}
}
\subfloat[][]{
	\includegraphics[scale=0.78]{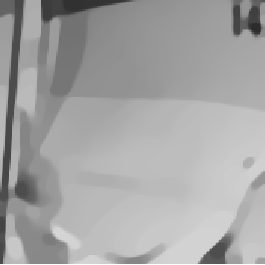}
}
\subfloat[][]{
	\includegraphics[scale=0.67]{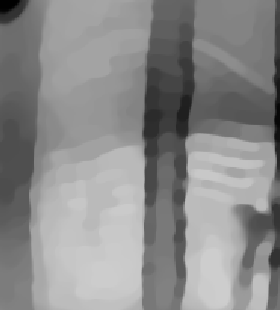}
}
\\
\subfloat[][]{
	\includegraphics[scale=0.47]{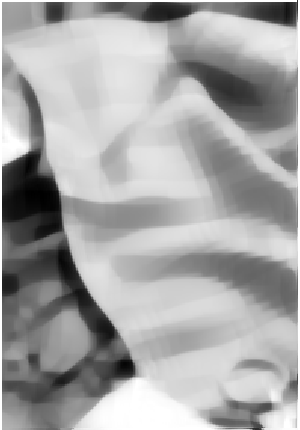}
}
\subfloat[][]{
	\includegraphics[scale=0.78]{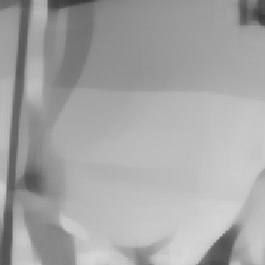}
}
\subfloat[][]{
	\includegraphics[scale=0.67]{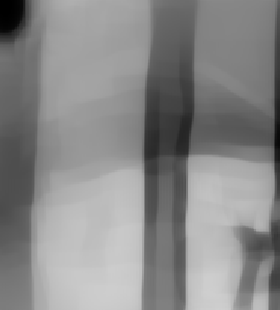}
}

\caption{\it 
 Test images from \cite{gra} (1st row): clutter (a), sail (b), fish.
Enlarged details from cartoons computed via $TV-\ell^{2}$ (2nd row), $TV-G$ (3rd row), $TV$-Hilbert (4th row) and our LsR-model (5th row).
}
\label{fig:c_t_2}
\end{figure}

Figures~\ref{fig:c_t_1} and~\ref{fig:c_t_2} show that the weakly factoring LsR filters lead to results  comparable in texture removal to classic models in the literature. Since the LsR filters come with higher flexibility (e.g. the additional parameter $\gamma>\frac{1}{2}$ and the number of scales can be increased), they have the potential to be better trained for a specific task.

\section{Outlook}

In this paper we have introduced a general intersection point problem which features a solution under (CPC) and weakly factoring filters. We have explored some consequences of these conditions. It is an open problem whether additional conditions are necessary to establish convergence and to obtain a unique intersection point. Both of which are granted if weakly factoring is replaced with strongly factoring. 
For the weakly factoring filter families introduced, we have observed numerical convergence for all of the test images considered. 

The weakly factoring filters satisfying the~(\ref{eq:CPC}), introduced, serve as a proof of concept. Due to their broad range of parameters, they can be used for (deep) learning the goals of denoising and cartoon-texture decomposition, as we have touched in our applications. In particular when a loss-function for a specific task is available, minimizing this loss over a larger class of problems is desirable. 

 
In current research we use the filters proposed in our applications as initial parameters that are fed into learning architectures, where filter families $\uB$ and corresponding factor families $\uY$ are optimized over for supervised and unsupervised learning.

\section*{Acknowledgment}
The first and the last author gratefully acknowledge funding by the DFG GRK 2088.

\addcontentsline{toc}{chapter}{Bibliography}
\bibliography{referencesRR}
\bibliographystyle{alpha}
\newpage

\section{Appendix I - Some Basic Mapping Degree Theory}
\label{app:one}

Let $\cK\in \R^n$  be the closure of an open bounded set.

\begin{PropDef}[\cite{OR09}, §{}IV, Prop. and Def. 1.1]\label{def:degree_A}
Let $g:\cK\to\R^{n}$ be a differentiable mapping and $p\notin{}g(\parO)$ a regular value (with preimage of finite cardinality), then the \emph{degree} $\text{deg}(g,\cK,p)$ is defined as:
\begin{equation}\label{eq:degree}
\text{deg}(g,\cK,p):=\sum_{x\in{}g^{-1}(p)}\text{sign}\lb\text{det}\lb{}D_{x}g\rb\rb\,,
\end{equation} 
where the sum is $0$ in case of $g^{-1}(p)=\emptyset$ and $D_{x}g$ is the Jacobian matrix at point $x\in\cK$.
\end{PropDef}

The notion of degree can be extended to continuous functions.

\begin{PropDef}[\cite{OR09}, §{}IV, Prop. and Def. 2.1]\label{def:degree_B}
Let $g:\cK\to\R^{n}$ be a continuous mapping and $p\notin{}g(\parO)$. Then there exists a differentiable mapping $h:\cK\to\R^{n}$ such that $\lab\lab{}g-h\rab\rab_{\infty}\leq{}d(p,g(\parO))$ and $p$ is a regular value of $h$. For every such $h$ the \emph{degree} $d(h,\cK,p)$ is well-defined and independent of the choice of $h$, so we can define the degree with respect to $g$ as:
\begin{displaymath}
\text{deg}(g,\cK,p):=\text{deg}(h,\cK,p)=\sum_{x\in{}h^{-1}(p)}\text{sign}\lb\text{det}\lb{}D_{x}h\rb\rb\,.
\end{displaymath} 
\end{PropDef}

The degree of a continuous function is homotopy-invariant. 

\begin{Lem}[\cite{OR09} §IV, Prop. 2.4]\label{lem:homoin}
Let $g$ and $h$ be continuous mappings from $\cK$ to $\R^{n}$ and let $H:[0,1]\times\cK\to\R^{n}$ be a homotopy between $g$ and $h$, i.e. a continuous mapping such that $H(0,\cdot{})=g$ and $H(1,\cdot{})=h$, such that for all $t\in[0,1]$ we have $p\notin{}H(t,\parO)$, then
\begin{displaymath}
\text{deg}(f,\cK,p)=\text{deg}(h,\cK,p)\,.
\end{displaymath}
\end{Lem}

We will need the following Corollary.

\begin{Cor}[\cite{OR09}, §{}IV, Cor. 2.5 (2)]\label{cor:notzero}
Let $g:\cK\to\R^{n}$ be a continuous mapping and let $p\notin{}g(\parO)$, if $\text{deg}(g,\cK,p)\neq{}0$ then $p\in\text{Im}(g)$.
\end{Cor}

\section{Appendix II - Proof of Lemma~\ref{lem:intersection}}\label{app:intersection}

W.l.o.g. assume that $\Xi,\Omega$ are linear subspaces, intersecting at the origin $v=0$, and their sum spans $\R^D$.

\paragraph{We first show that $\Xi\cap \phi(\Omega) \neq \emptyset$.}
Suppose that $1 \leq \dim(\Omega) = m < D$, $\dim(\Xi) = D-m$ and choose unit length, pairwise orthogonal vectors $h_1,\ldots,h_D\in \R^D$ such that the first $m$ are orthogonal to $\Xi$ and the latter $D-m$ span $\Xi$. With the matrices $H=(h_1,\ldots,h_m)$ and $S=(h_{m+1},\ldots,h_D)$ define the mappings
\begin{equation}\label{eq:Psi}
\begin{aligned}
\Psi_{1}&:{}\Omega\to\R^{m}\,;\omega\mapsto{}H^{T}\omega\\
\Psi_{2}&:{}\Omega\to\R^{m}\,;\omega\mapsto{}H^{T}\phi(\omega)\,,
\end{aligned}
\end{equation}
and note that $H^Tu =0 \Leftrightarrow u \in \Xi$, as well as  
\begin{equation*}
\begin{aligned}
\lab\lab{}S^{T}\omega\rab\rab^{2}&=\sum_{j=m+1}^{n}\lb{}h_{j}^{T}\omega\rb^{2} ={}
\max_{\xi\in{}\Xi,\lab\lab{}\xi\rab\rab=1}\lb{}\xi^{T}\omega\rb^{2}\\
&\leq{}\lab\lab{}\omega\rab\rab^{2}\underset{\begin{array}{c}\xi\in{}\Xi:\lab\lab{}\xi\rab\rab=1\\
\widetilde{\omega}\in{}\Omega:\lab\lab{}\widetilde{\omega}\rab\rab=1\end{array}}{\max}\lb{}\xi^{T}\widetilde{\omega}\rb^{2}
=\lab\lab{}\omega\rab\rab^{2}\cos^{2}(\theta)\,,
\end{aligned}
\end{equation*}
where $\theta\in[0,1)$ is the first principal angle between $\Xi$ and $\Omega$.
Hence, we have for all $\omega\in \Omega$ that
\begin{equation} \label{eq:alltogether}
\lab\lab{}\Psi_{1}(\omega)\rab\rab^{2}=
\omega^T(H|S)\left(\begin{array}{c}H^T\\S^T\end{array}\right)\omega - \omega^TSS^T\omega 
\geq{}\lab\lab{}\omega\rab\rab^{2}\lb{}1-
\cos^{2}(\theta)\rb=\lab\lab{}\omega\rab\rab^{2}\sin^{2}(\theta)\,.
\end{equation}
Thus, with the closed ball
\begin{displaymath}
\cK:=\lbb{}\omega\in{}\Omega:\lab\lab{}\omega\rab\rab\leq{}C\lb{}1+\frac{2}{\sin(\theta)}\rb\rbb\,,
\end{displaymath}
in $\Omega$ and  $\partial\cK$, its boundary within $\Omega$, we have in particular for $a\in{}\partial\cK$ that 
\begin{equation}\label{eq:PsiB}
\begin{aligned}
\lab\lab{}\Psi_{1}(a)\rab\rab\geq{}\lab\lab{}a\rab\rab\sin(\theta)=C\lb{}\sin(\theta)+2\rb\geq{}2C\,.
\end{aligned}
\end{equation}
Moreover, by hypothesis on $\phi$ we have for all $a\in\Omega$, in particular for all $a\in{}\partial\cK$, that
\begin{equation}\label{eq:fromW}
\begin{aligned}
\lab\lab{}\Psi_{2}(a)-\Psi_{1}(a)\rab\rab&=\lab\lab{}H^{T}\phi(a)-H^{T}a\rab\rab 
\leq{}\lab\lab{}\phi(a)-a\rab\rab\leq{}C\,,
\end{aligned}
\end{equation}
by orthonormality of the columns of $H$.

Now, introducing a homotopy
\begin{displaymath}
\nu:[0,1]\times{}\Omega\to\R^{m}\,,\quad{}(t,w)\mapsto{}t\Psi_{1}(w)+(1-t)\Psi_{2}(w)\,,
\end{displaymath}
between $\Psi_{1}$ and $\Psi_{2}$, exploiting~(\ref{eq:PsiB}) and~(\ref{eq:fromW}), we have for all $a\in{}\partial\cK$ and $t\in [0,1]$ that
\begin{displaymath}
\begin{aligned}
\lab\lab{}\nu(t,a)\rab\rab&=\lab\lab{}t\Psi_{1}(a)+(1-t)\Psi_{2}(a)\rab\rab\\
&=\lab\lab{}\Psi_{1}(a)-(1-t)\lb{}\Psi_{1}(a)-\Psi_{2}(a)\rb\rab\rab\\
&\geq{}\lab\lab{}\Psi_{1}(a)\rab\rab-\lab\lab{}(1-t)\lb{}\Psi_{1}(a)-\Psi_{2}(a)\rb\rab\rab\geq{}C\,,
\end{aligned}
\end{displaymath}
yielding
\begin{equation}\label{eq:0_not_in_A}
\nexists{}\,(t,a)\in{}[0,1]\times{}\partial\cK\text{ such that }\nu(t,a)=0\,.
\end{equation}

Finally, let $B=(b_1,\ldots,b_m)$ with an orthonormal basis $b_{1},b_{2},\dots{},b_{m}$ of $\Omega$ and define the isomorphism
\begin{displaymath}
f:\R^{m}\to{}\Omega\,,\quad{}x \mapsto Bx\,. 
\end{displaymath}
Then
\begin{displaymath}
\nu_{f}:[0,1]\times{}f^{-1}(\cK),\qquad{}(t,x)\mapsto{}\nu(t,f(x))\,,
\end{displaymath}
is a homotopy between the isomorphism $\Psi_{1}\circ{}f : \R^m \to \R^m$ and $\Psi_{2}\circ{}f$, and by (\ref{eq:0_not_in_A}) we have
\begin{displaymath}
\nexists{}\,(t,a)\in{}[0,1]\times{}\lb{}f^{-1}(\partial\cK)\rb\text{ such that }\nu_{f}(t,a)=0\,.
\end{displaymath}
Hence, we can apply Lemma~\ref{lem:homoin} to obtain
\begin{displaymath}
\begin{aligned}
\text{deg}\lb{}\Psi_{2}\circ{}f,f^{-1}(\cK),0\rb&=\text{deg}\lb{}\Psi_{1}\circ{}f,f^{-1}(\cK),0\rb\\
&=\text{sign}\lb{}\text{det}(D_{0}[\Psi_{1}\circ{}f])\rb\\
&=\text{sign}\lb\text{det}(D_{0}[\Psi_{1}]D_{0}[f])\rb\\
&=\text{sign}\lb\text{det}(H^{T}B)\rb\neq 0\,. 
\end{aligned}
\end{displaymath}
By Corollary~\ref{cor:notzero} we obtain the existence of at least one $x\in\R^{m}$ such that $H^{T}(\phi{}\circ{}f)(x)=0$, yielding that there must exist at least one $\omega\in\phi(\Omega)$ such that $H^{T}\omega=0$. By construction of $H$ such an $\omega$ is in $\Xi$ proving the first assertion.

\paragraph{Next, we show the asserted inequality.} Let $u=\phi(\omega)\in{}\Xi\cap\phi(\Omega)$ with $\omega\in\Omega$. Then, we have $||u-\omega||\leq{}C$ by hypothesis on $\phi$ and 
by (\ref{eq:fromW}) that
\begin{displaymath}
||\Psi_{1}(\omega)||=||H^{T}\omega||=||H^{T}\omega-H^{T}u||
=||\Psi_{1}(\omega)-\Psi_{2}(\omega)||\leq{}C\,.
\end{displaymath}
Furthermore, by~(\ref{eq:alltogether}) we have
\begin{displaymath}
||\omega||\leq{}\frac{C}{\sin(\theta)}\,,
\end{displaymath}
yielding finally,
\begin{displaymath}
||u||\leq{}||u-\omega||
+||\omega||\leq{}C\lb{}1+\frac{1}{\sin(\theta)}\rb\,.
\end{displaymath}

\section{Appendix III - Detail to 
Separating Cartoon and Texture}\label{sep-cartoon-texture:appendix}

Recalling (\ref{omega-def:eq}) and (\ref{eq:Spline}) 
introduce the \emph{refinement filter function} 
\begin{displaymath}
h_{\gamma}(x,y):=\frac{2\widehat{f}_{\gamma}\lb{}-2x,-2y\rb}{\widehat{f}_{\gamma}\lb{}-x,-y\rb}\,,
\end{displaymath}
see~\cite[Eq.(29)]{VBU05},
obtain for a scale $j\in\lbb{}0,1,\dots{},J-1,J\rbb{}$  via dyadic sub-sampling (cf. \cite{Mal98}) the primal wavelet frames $T^{\text{primal}}_{j,s}\in\SpM$, for $s=1,2,3$,
\begin{equation}\label{eq:Tprimal}
\begin{aligned}
T^{\text{primal}}_{j,1}[k,\ell]&=\frac{1}{2}\exp\lb{}-i(2^{j-1}\omega_{k}+\pi)\rb{}h_{\gamma}(2^{j-1}\omega_{k}+\pi,2^{j-1}\omega_{\ell})\\
&\quad{}\times{}a_{\gamma}\lb{}2^{j-1}\omega_{k}+\pi, 2^{j-1}\omega_{\ell}\rb{}\widehat{f}_{\gamma}\lb{}2^{j-1}\omega_{k},2^{j-1}\omega_{\ell}\rb{}\,,\\
T^{\text{primal}}_{j,2}[k,\ell]&=\frac{1}{2}\exp\lb{}-i(2^{j-1}\omega_{k}+\pi)\rb{}h_{\gamma}(2^{j-1}\omega_{k},2^{j-1}\omega_{\ell}+\pi)\\
&\quad{}\times{}a_{\gamma}\lb{}2^{j-1}\omega_{k}, 2^{j-1}\omega_{\ell}+\pi\rb{}\widehat{f}_{\gamma}\lb{}2^{j-1}\omega_{k},2^{j-1}\omega_{\ell}\rb{}\,,\\
T^{\text{primal}}_{j,3}[k,\ell]&=\frac{1}{2}\exp\lb{}-i(2^{j-1}\omega_{k}+\pi)\rb{}h_{\gamma}(2^{j-1}\omega_{k}+\pi,2^{j-1}\omega_{\ell}+\pi)\\
&\quad{}\times{}a_{\gamma}\lb{}2^{j-1}\omega_{k}+\pi, 2^{j-1}\omega_{\ell}+\pi\rb{}\widehat{f}_{\gamma}\lb{}2^{j-1}\omega_{k},2^{j-1}\omega_{\ell}\rb{}\,.
\end{aligned}
\end{equation}
and their dual counterparts 
\begin{equation}\label{eq:Tdual}
\begin{aligned}
T^{\text{dual}}_{j,1}[k,\ell]&=\frac{1}{2}\exp\lb{}-i(2^{j-1}\omega_{k}+\pi)\rb{}\frac{h_{\gamma}(2^{j-1}\omega_{k}+\pi,2^{j-1}\omega_{\ell})}{a_{\gamma}\lb{}2^{j}\omega_{k},2^{j}\omega_{\ell}\rb}\\
&\quad{}\times{}\frac{\widehat{f}_{\gamma}\lb{}2^{j-1}\omega_{k},2^{j-1}\omega_{\ell}\rb}{a_{\gamma}\lb{}2^{j-1}\omega_{k}, 2^{j-1}\omega_{\ell}\rb{}}\,,\\
T^{\text{dual}}_{j,2}[k,\ell]&=\frac{1}{2}\exp\lb{}-i(2^{j-1}\omega_{k}+\pi)\rb{}\frac{h_{\gamma}(2^{j-1}\omega_{k},2^{j-1}\omega_{\ell}+\pi)}{a_{\gamma}\lb{}2^{j}\omega_{k},2^{j}\omega_{\ell}\rb}\\
&\quad{}\times{}\frac{\widehat{f}_{\gamma}\lb{}2^{j-1}\omega_{k},2^{j-1}\omega_{\ell}\rb}{a_{\gamma}\lb{}2^{j-1}\omega_{k},2^{j-1}\omega_{\ell}\rb{}}\,,\\
T^{\text{dual}}_{j,3}[k,\ell]&=\frac{1}{2}\exp\lb{}-i(2^{j-1}\omega_{k}+\pi)\rb{}\frac{h_{\gamma}(2^{j-1}\omega_{k}+\pi,2^{j-1}\omega_{\ell}+\pi)}{a_{\gamma}\lb{}2^{j}\omega_{k},2^{j}\omega_{\ell}\rb}\\
&\quad{}\times{}\frac{\widehat{f}_{\gamma}\lb{}2^{j-1}\omega_{k},2^{j-1}\omega_{\ell}\rb}{a_{\gamma}\lb{}2^{j-1}\omega_{k}, 2^{j-1}\omega_{\ell}\rb{}}\,.\\
\end{aligned}
\end{equation}
%

Further we link the isotropic polyharmonic B-splines to the $Z$-th order \emph{Riesz transform} in order to model directionality, cf. \cite{USV09} and \cite[Eqns. (4) and (10)]{UV10}, and define for $z=1,2,\dots{},Z$ the matrices $R_{z}\in\SpM$ by
\begin{displaymath}
R_{z}[k,\ell]:=(-i)^{Z}\sqrt{\frac{Z!}{z!(Z-z)!}}\frac{\omega_{k}^{z}\omega_{\ell}^{Z-z}}{\lb\omega_{k}^{2}+\omega_{\ell}^{2}\rb^{\frac{Z}{2}}}\,.
\end{displaymath}

Now we are ready to define $\uB$ and $\uBT$.
For $j=0,\dots{},J$, $z=1,\dots{},Z$, $s=1,2,3$ define
\begin{displaymath}
\begin{aligned}
\widehat{B}_{j,z,s}[k,\ell]&=T^{\text{primal}}_{j,s}[k,\ell]R_{z}[k,\ell]\\
\widehat{\widetilde{B}}_{j,z,s}[k,\ell]&=T^{\text{dual}}_{j,s}[k,\ell]R_{z}[k,\ell]\,.
\end{aligned}
\end{displaymath}
From~(\ref{eq:Tprimal}) and~(\ref{eq:Tdual}) we have at once the factor matrix-family of $(\uB,\uBT)$ 
\begin{displaymath}
\begin{aligned}
\widehat{Y}_{j,z,1}[k,\ell]&=\lb{}a_{\gamma}(2^{j-1}\omega_{k}+\pi,2^{j-1}\omega_{\ell})a_{\gamma}(2^{j-1}\omega_{k},2^{j-1}\omega_{\ell})a_{\gamma}(2^{j}\omega_{k},2^{j}\omega_{\ell})\rb{}^{-1}\,,\\
\widehat{Y}_{j,z,2}[k,\ell]&=\lb{}a_{\gamma}(2^{j-1}\omega_{k},2^{j-1}\omega_{\ell}+\pi)a_{\gamma}(2^{j-1}\omega_{k},2^{j-1}\omega_{\ell})a_{\gamma}(2^{j}\omega_{k},2^{j}\omega_{\ell})\rb{}^{-1}\,,\\
\widehat{Y}_{j,z,3}[k,\ell]&=\lb{}a_{\gamma}(2^{j-1}\omega_{k}+\pi,2^{j-1}\omega_{\ell}+\pi)a_{\gamma}(2^{j-1}\omega_{k},2^{j-1}\omega_{\ell})a_{\gamma}(2^{j}\omega_{k},2^{j}\omega_{\ell})\rb{}^{-1}\,.
\end{aligned}
\end{displaymath}
Since $\widehat{Y}_{j_{1},z_{1},s_{1}}\neq{}\widehat{Y}_{j_{2},z_{2},s_{2}}$ for $j_{1}\neq{}j_{2}$ or $s_{1}\neq{}s_{2}$ we have that $(\uB,\uBT)$ is weakly but not strongly factoring.

In the numerical implementation, we approximate the infinite sum in~(\ref{eq:Auto}) by 
\begin{equation*}
a_{\gamma}(x,y):=\sum_{r,s\in\lbb{}-10,\dots{},10\rbb}\lb{}\widehat{f}_{\gamma}(x+2\pi{}r,y+2\pi{}s\rb^{2}\,.
\end{equation*}
Further due to correction by $H$ in equation (\ref{adjusted-filters:eq}), the resulting entries of the filters $\uB^{\text{cor}}$ and $\uBT^{\text{cor}}$, when transformed back to the spacial domain, have small, but non-zero imaginary parts. We cut these imaginary parts off. Since we defined filters with real entries in the frequency domain, by removing their imaginary part in the spatial domain, the matrix entries in the frequency domain remain real, so that all eigenvalues of the thus adjusted $\ucoBT^{\ast}\ucoB$ are real and the factor matrix-family has real entries, making the resulting filters weakly factoring. Since we have
\begin{eqnarray*}
\lefteqn{\lab{}\sum_{j,n,s}\overline{\widehat{\widetilde{B}}^{\text{cor}}_{j,n,s}[k,\ell]}\widehat{B}^{\text{cor}}_{j,n,s}[k,\ell]\rab}\\
&=&\lab\lb\sum_{j,n,s}\overline{\widehat{\widetilde{B}}_{j,n,s}[k,\ell]}\widehat{B}_{j,n,s}[k,\ell]\rb\lb\widehat{A}[k,\ell]+\sum_{j,n,s}\overline{\widetilde{B}_{j,n,s}[k,\ell]}\widehat{B}_{j,n,s}[k,\ell]\rb^{-1}\rab\\&\leq{}&1
\end{eqnarray*}
and observe that the above sum on the left-hand-side only gets smaller if the imaginary parts of $\uB,\uBT$ are removed, the resulting filters satisfy the~(\ref{eq:NEPC}). In all experiments conducted in this paper numerically the above inequality was strict, so that all filters used for the LbS model satisfied  the~(\ref{eq:CPC}).

\section{Appendix IV - More Cartoons}\label{app:cartoon}

\begin{figure}
\centering
\vspace{-2cm}
\subfloat[][]{
	\includegraphics[scale=0.33]{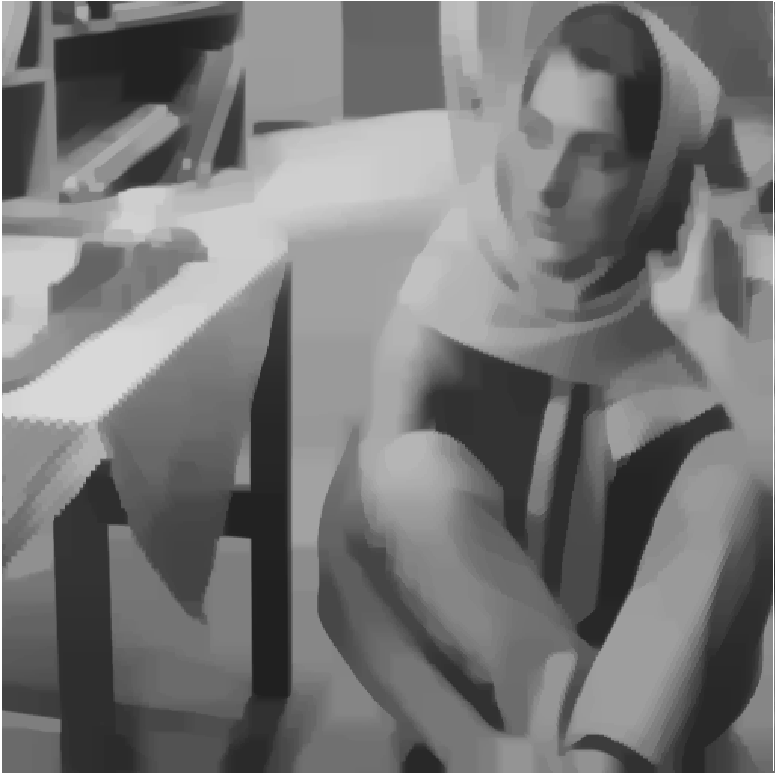}
}
\subfloat[][]{
	\includegraphics[scale=0.33]{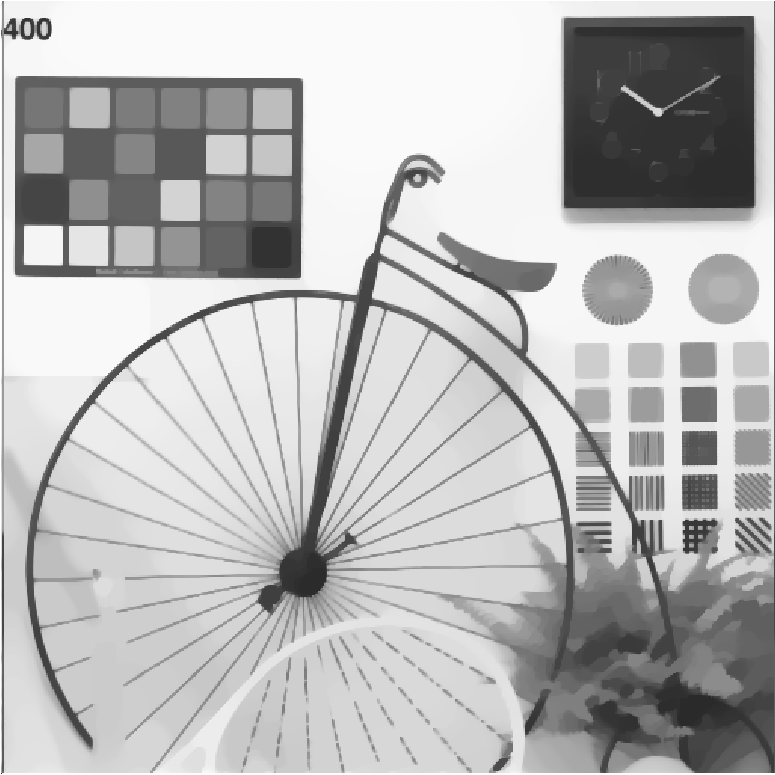}
}
\subfloat[][]{
	\includegraphics[scale=0.33]{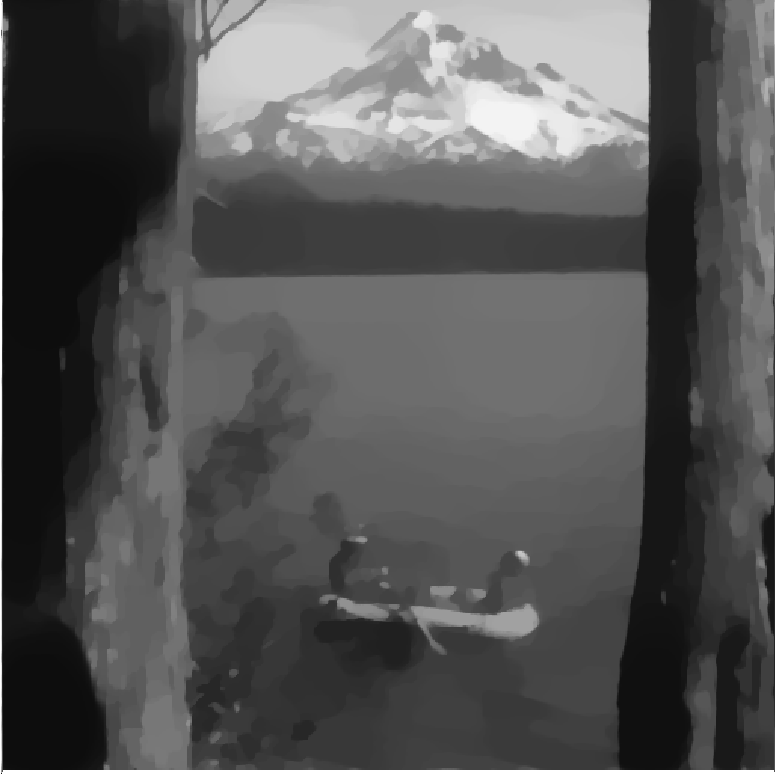}
}
\\
\subfloat[][]{
	\includegraphics[scale=0.33]{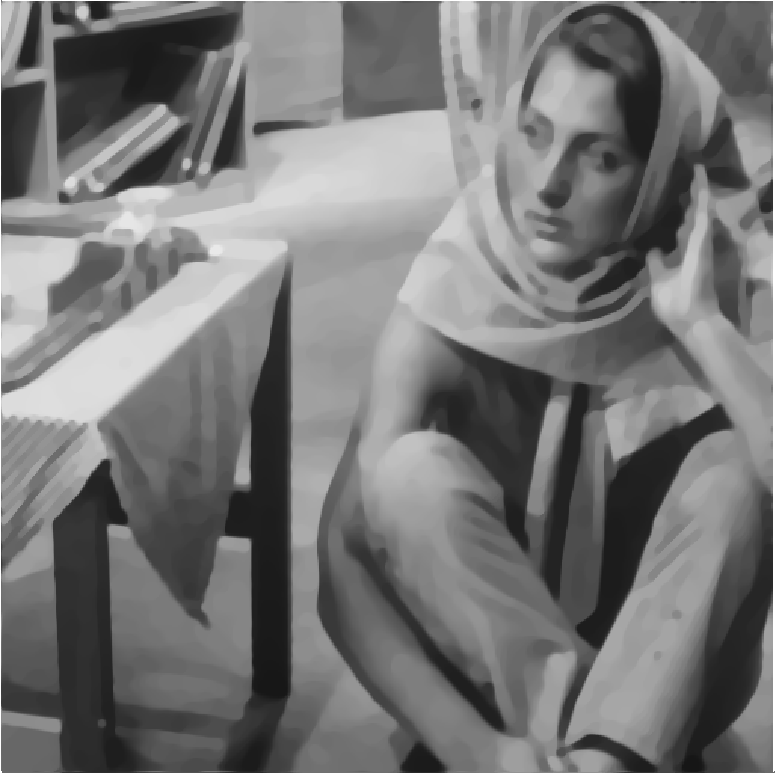}
}
\subfloat[][]{
	\includegraphics[scale=0.33]{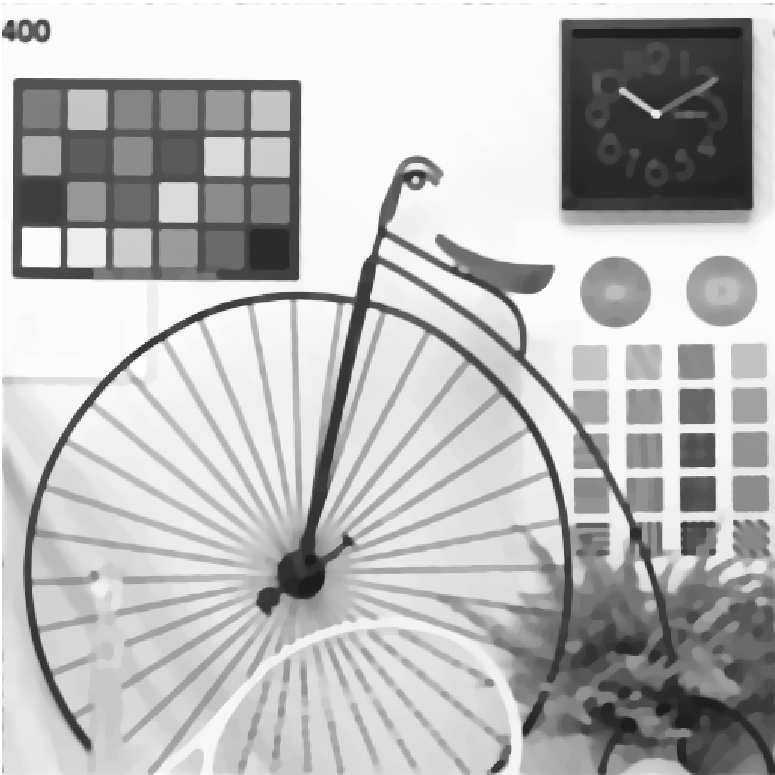}
}
\subfloat[][]{
	\includegraphics[scale=0.33]{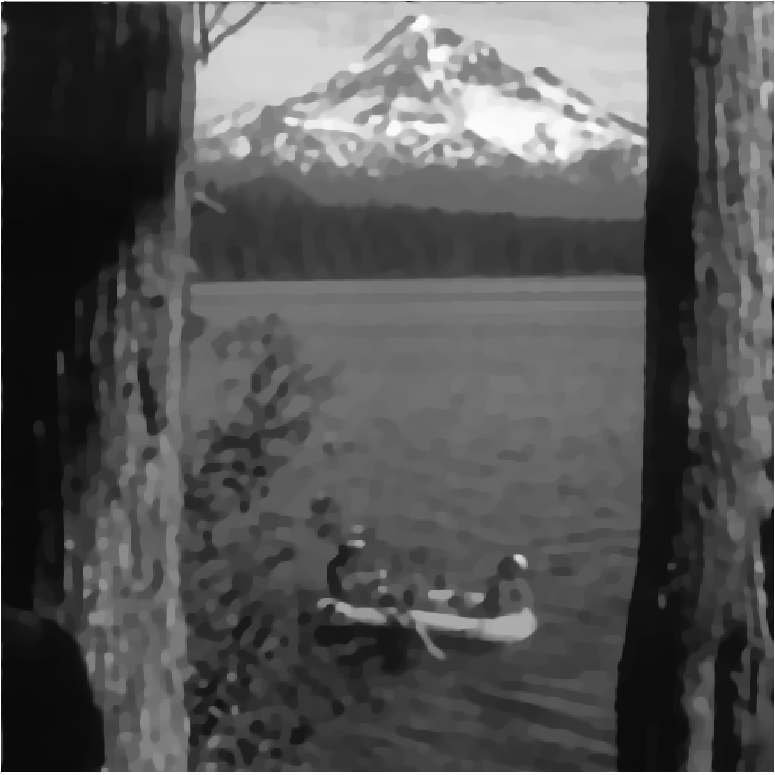}
}
\\
\subfloat[][]{
	\includegraphics[scale=0.33]{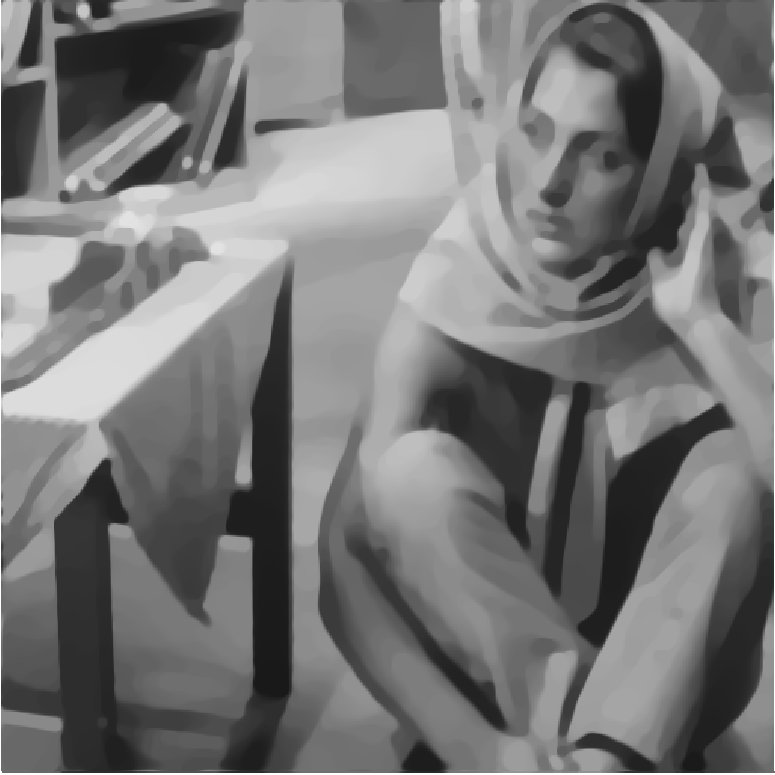}    
}
\subfloat[][]{
	\includegraphics[scale=0.33]{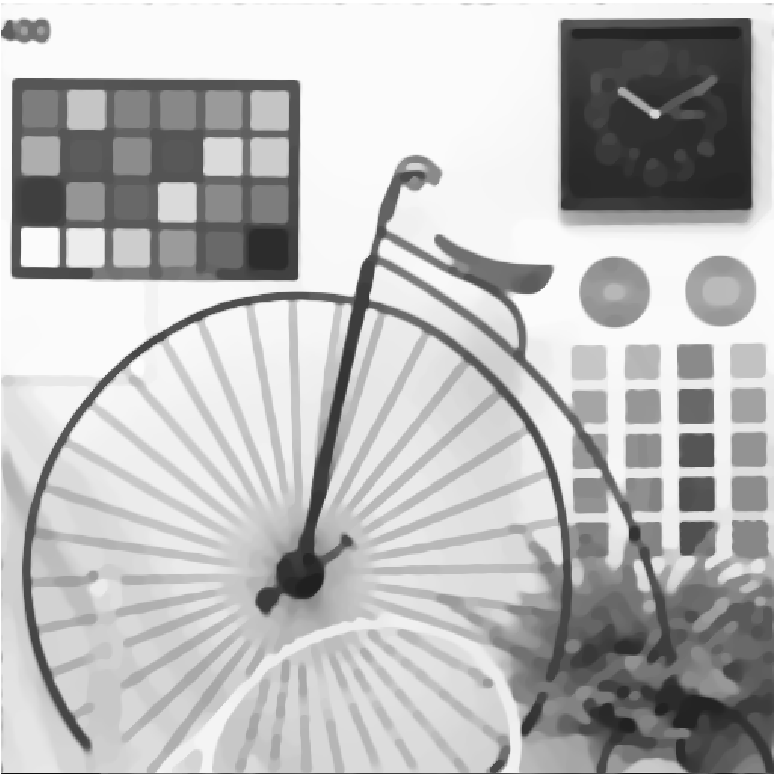}
}
\subfloat[][]{
	\includegraphics[scale=0.33]{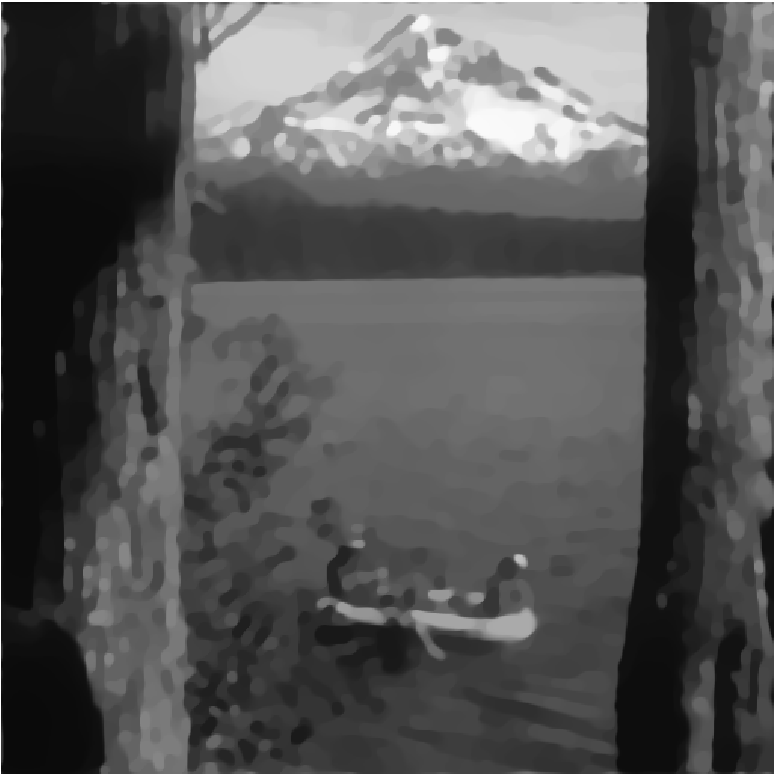}
}
\\
\subfloat[][]{
	\includegraphics[scale=0.33]{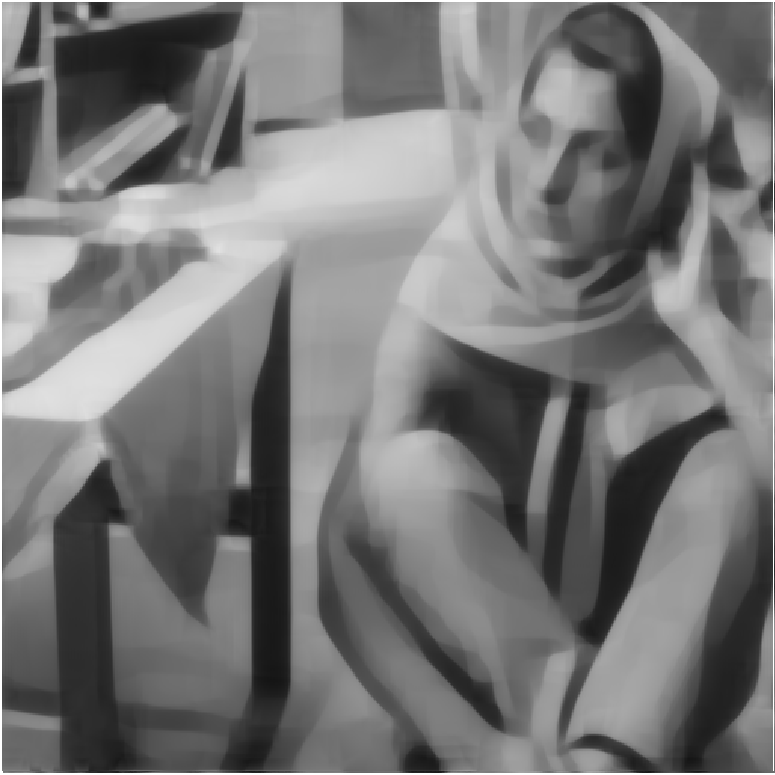}
}
\subfloat[][]{
	\includegraphics[scale=0.33]{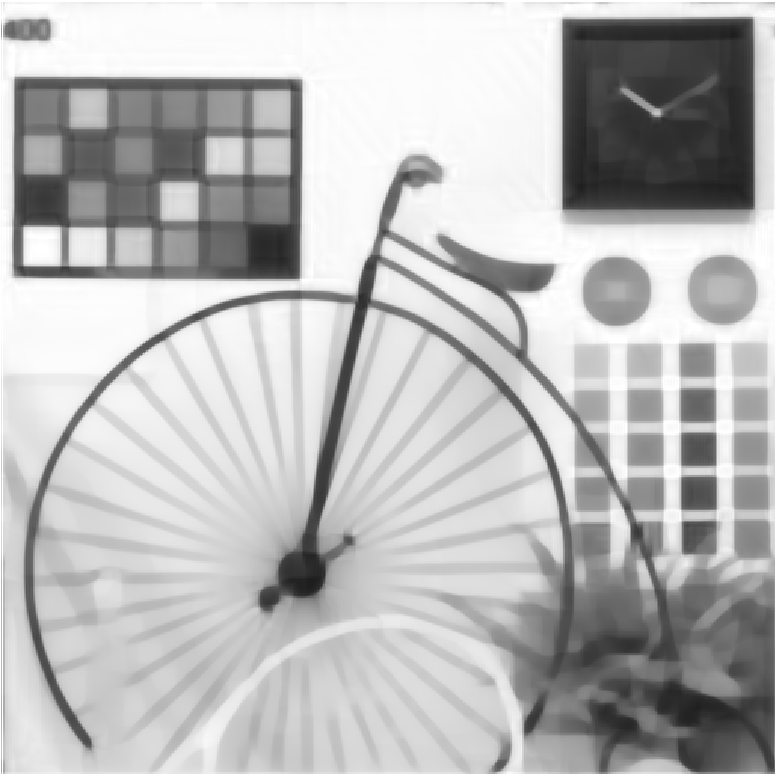}
}
\subfloat[][]{
	\includegraphics[scale=0.33]{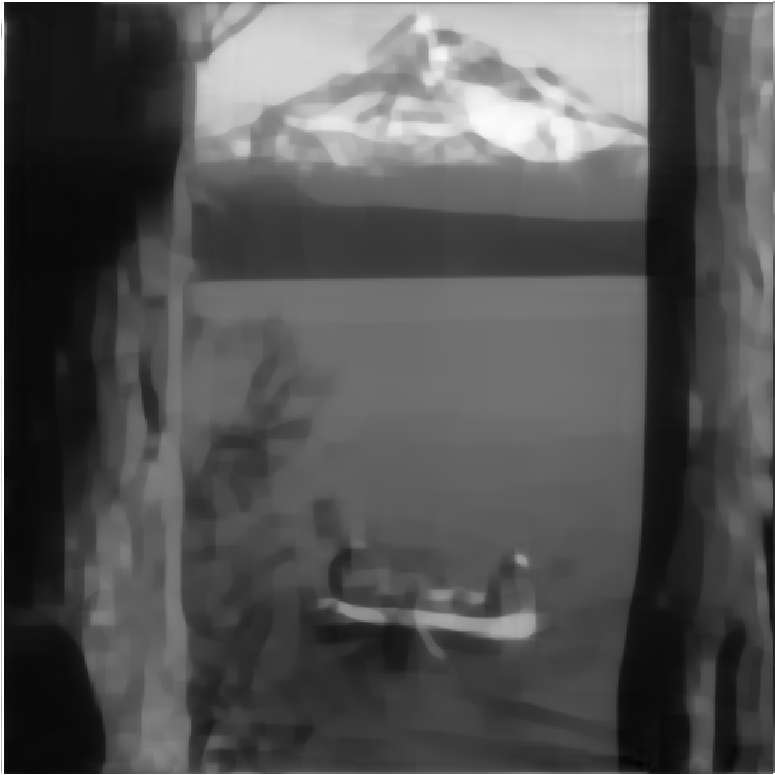}
}

\caption{Cartoons computed from the test images barbara, bicycle, lake from Figure \ref{fig:c_t_1} (see~\cite{gra} and ) via $TV-\ell^{2}$ (1st row: a-c), $OSV$ (2nd row: d-f), $TV$-Hilbert (3rd row: g-i) and the proposed LsR (4th row: j-l).} 
\label{fig:c_t_1_full}
\end{figure}

\begin{figure}
\centering
\vspace{-2cm}
\subfloat[][]{
	\includegraphics[scale=0.33]{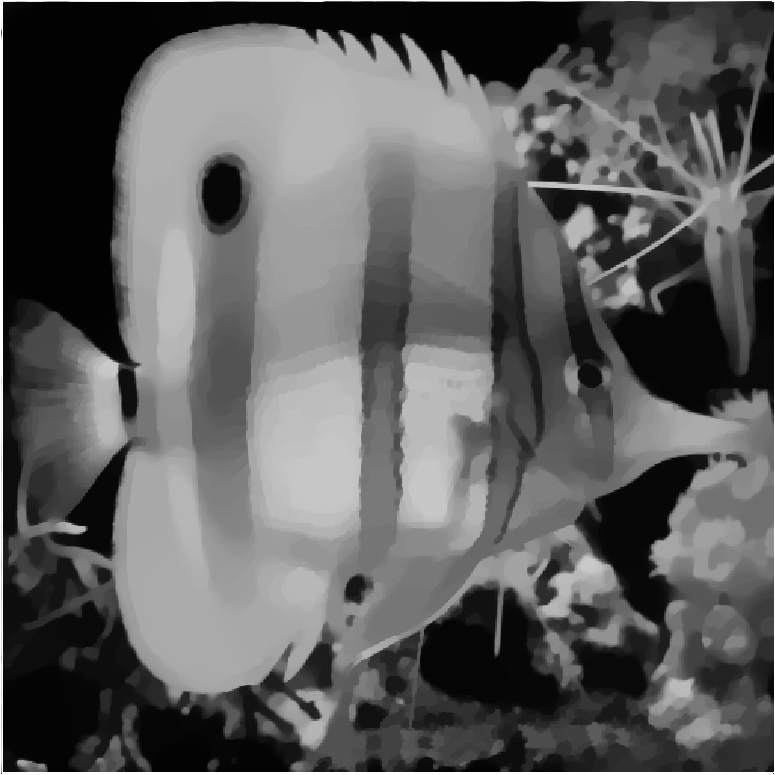}
}
\subfloat[][]{
	\includegraphics[scale=0.33]{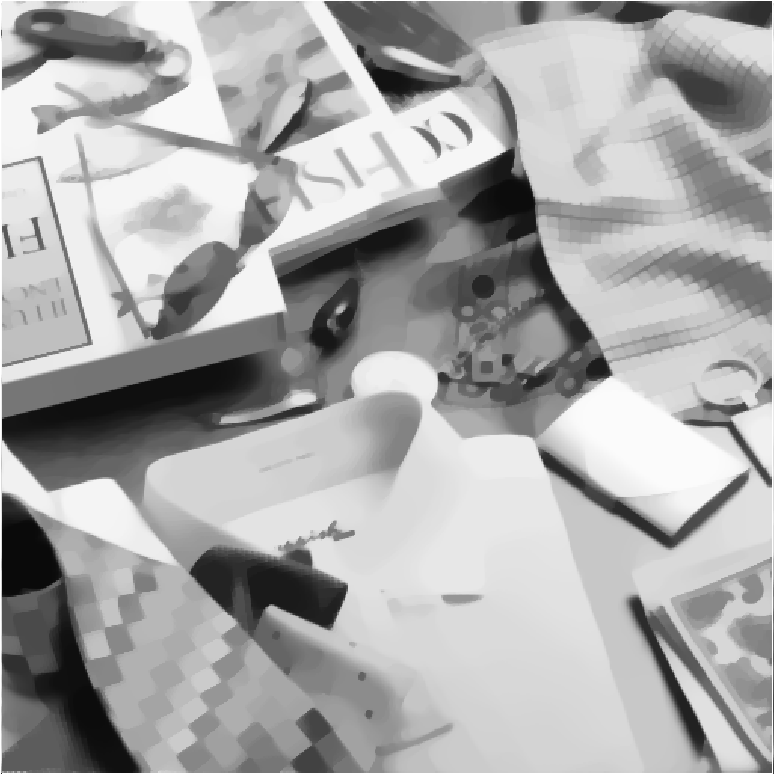}
}
\subfloat[][]{
	\includegraphics[scale=0.33]{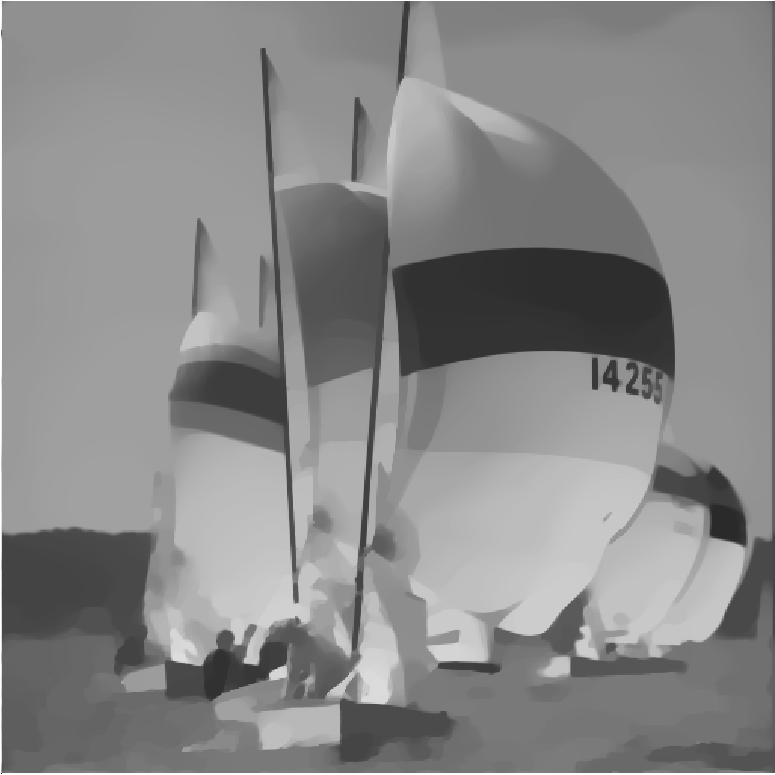}
}
\\
\subfloat[][]{
	\includegraphics[scale=0.33]{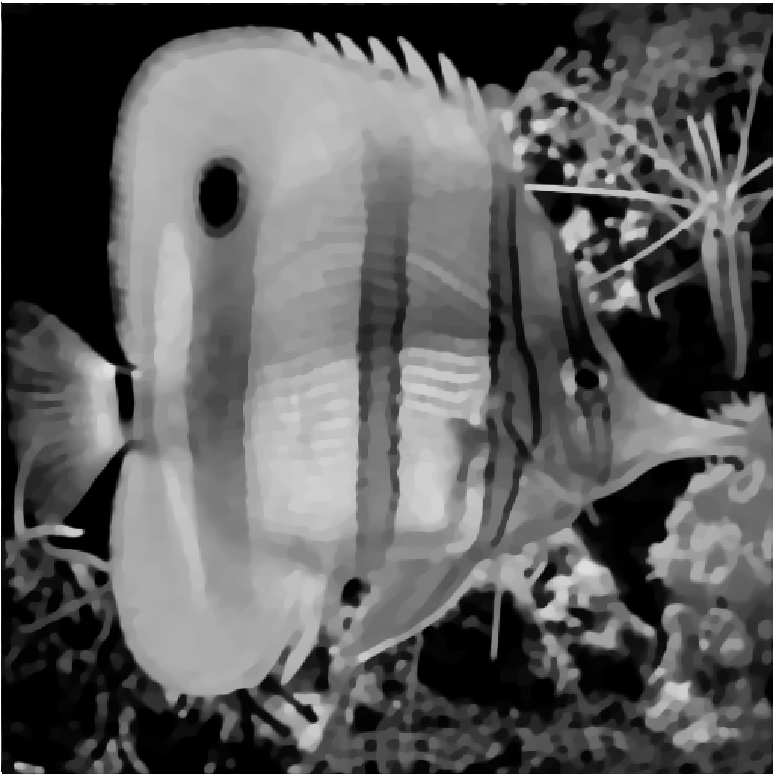}
}
\subfloat[][]{
	\includegraphics[scale=0.33]{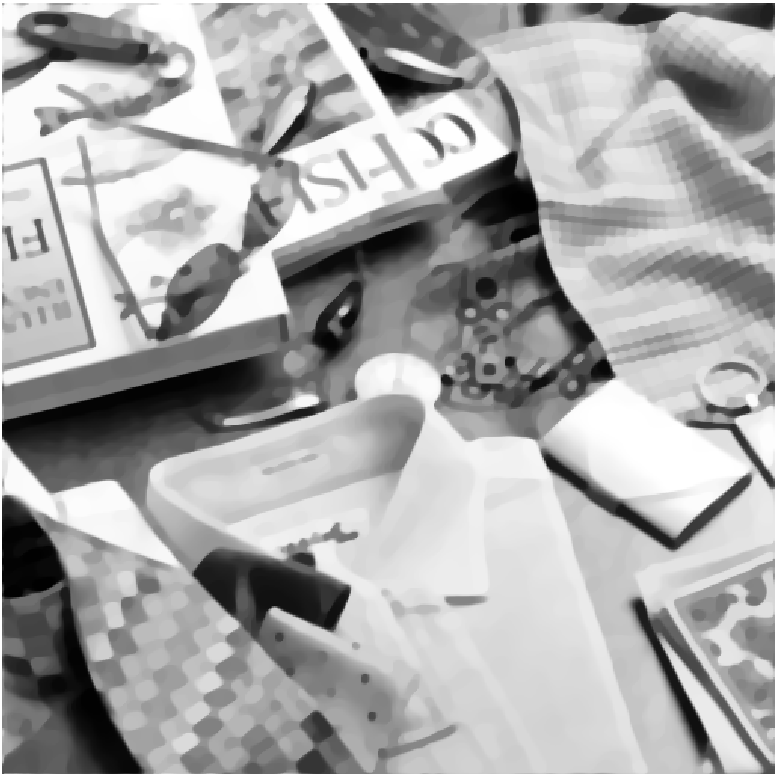}
}
\subfloat[][]{
	\includegraphics[scale=0.33]{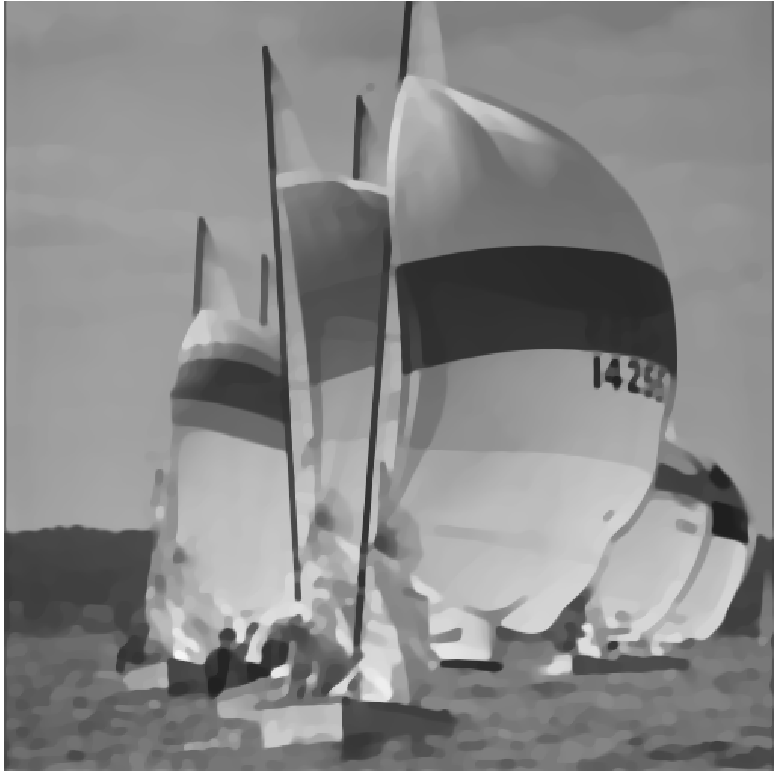}
}
\\
\subfloat[][]{
	\includegraphics[scale=0.33]{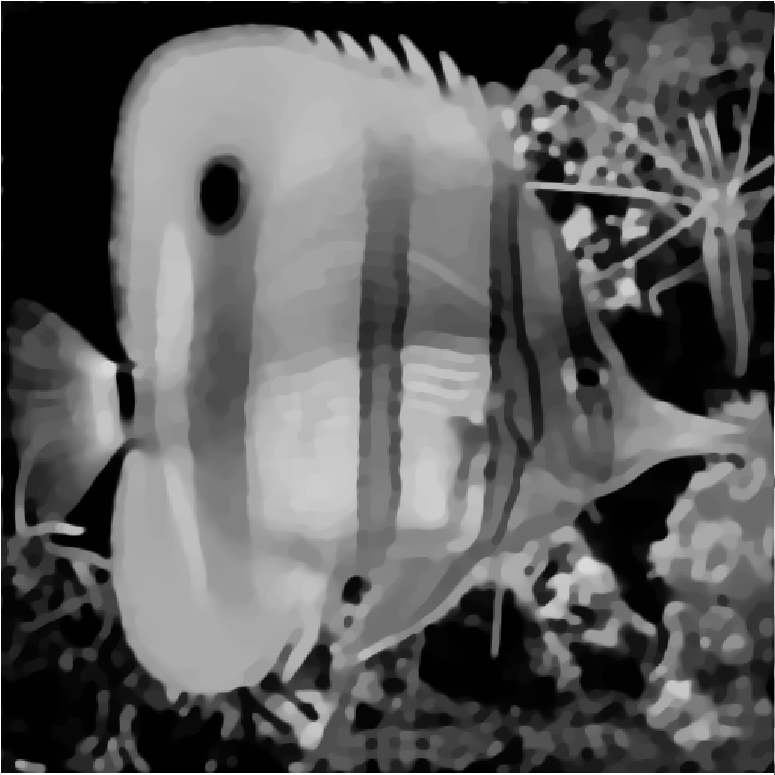}
}
\subfloat[][]{
	\includegraphics[scale=0.33]{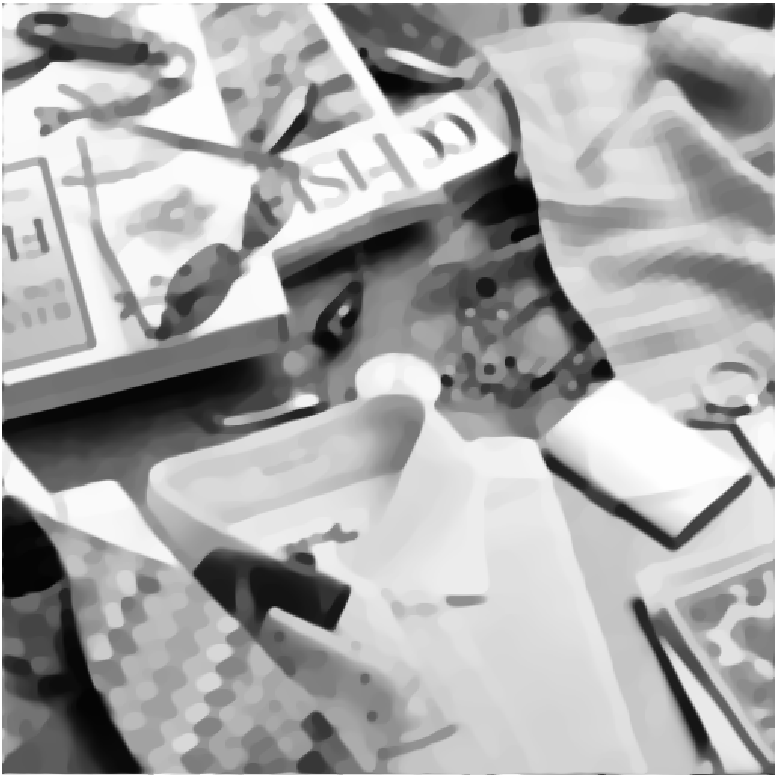}
}
\subfloat[][]{
	\includegraphics[scale=0.33]{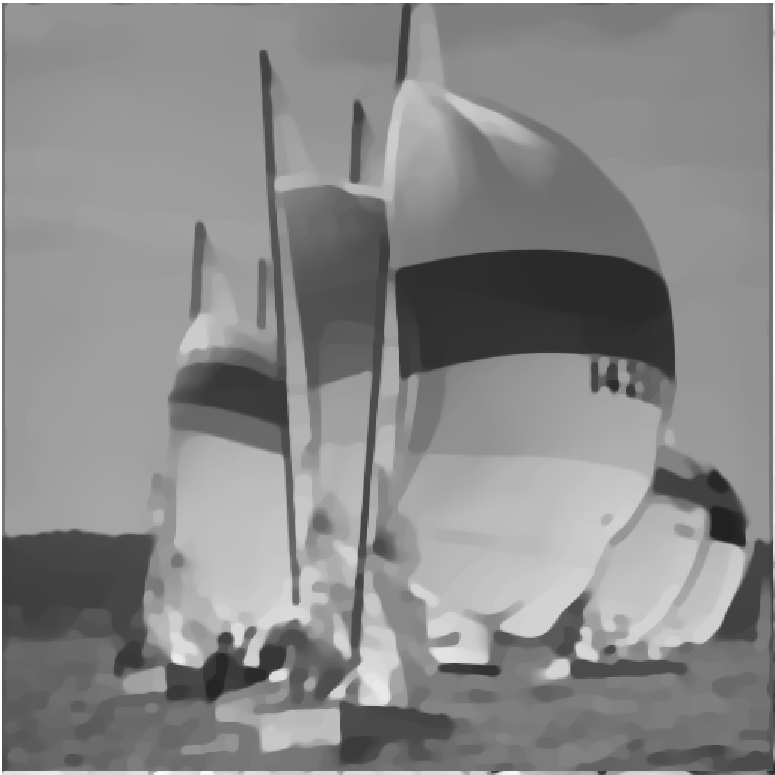}
}
\\
\subfloat[][]{
	\includegraphics[scale=0.33]{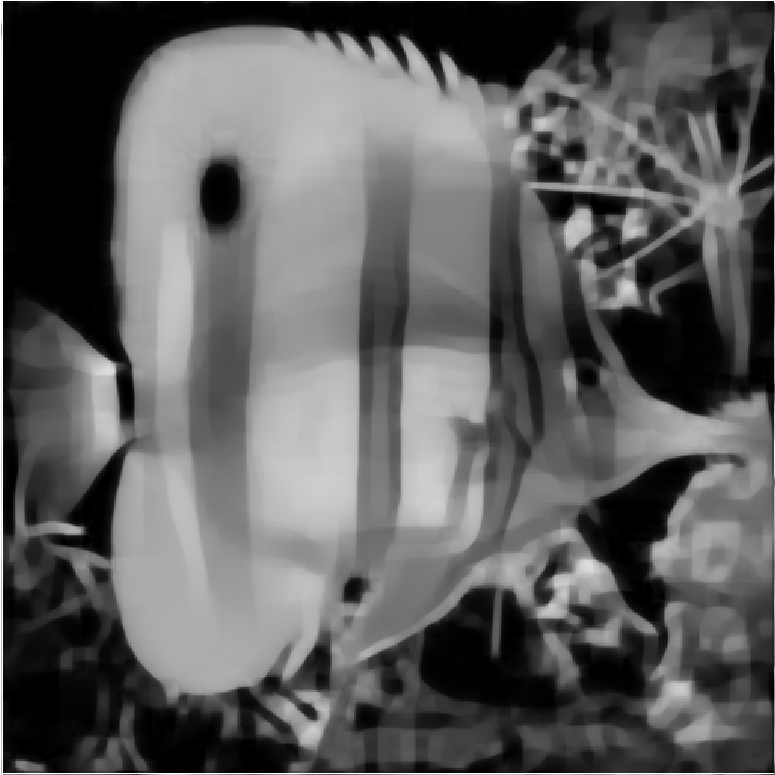}
}
\subfloat[][]{
	\includegraphics[scale=0.33]{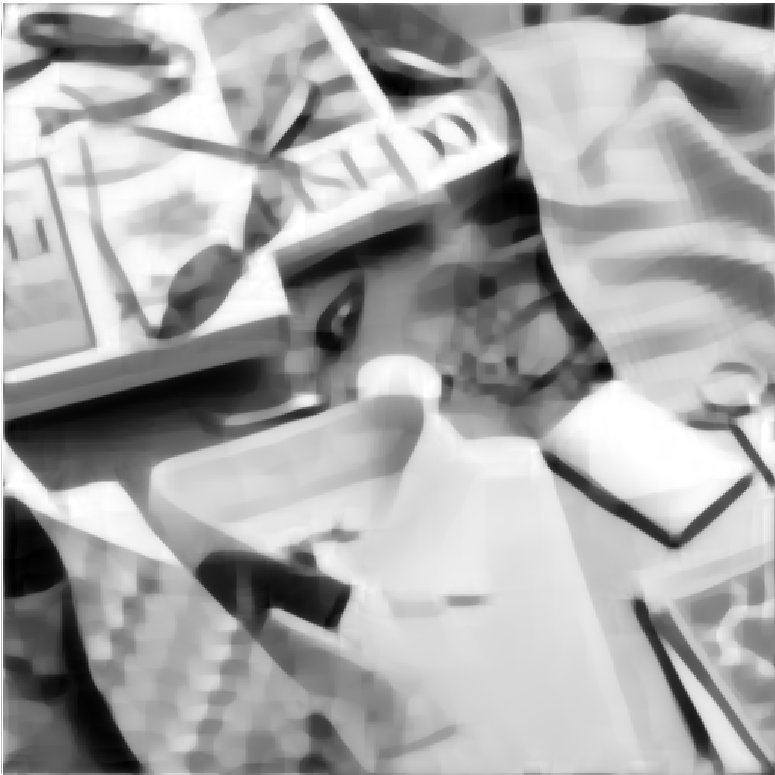}
}
\subfloat[][]{
	\includegraphics[scale=0.33]{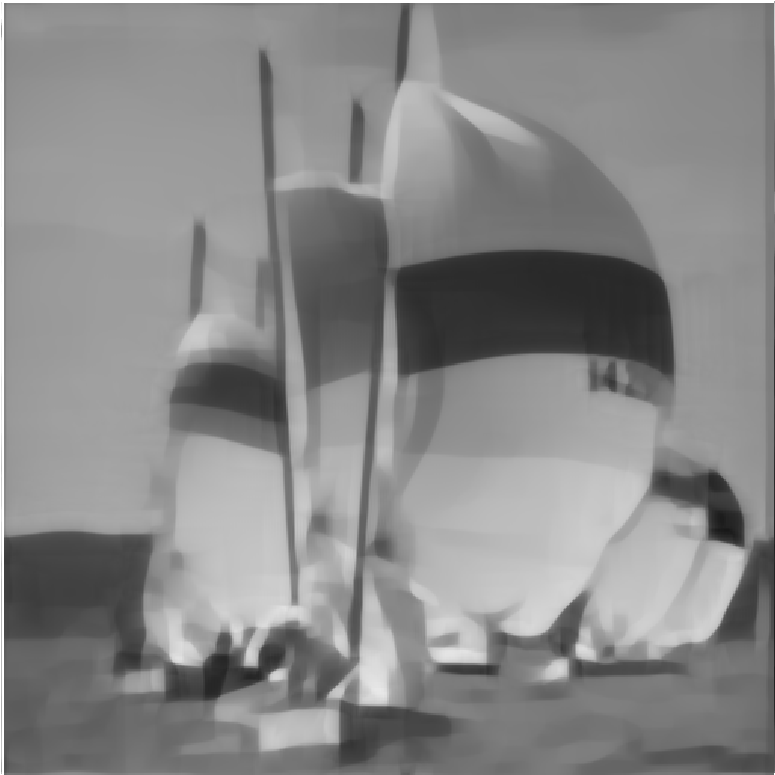}
}

\caption{Cartoons computed from the test fish, clutter, sail from Figure \ref{fig:c_t_2} (see~\cite{gra} and ) via $TV-\ell^{2}$ (1st row: a-c), $OSV$ (2nd row: d-f), $TV$-Hilbert (3rd row: g-i) and the proposed LsR (4th row: j-l).} 
\label{fig:c_t_2_full}
\end{figure}

\end{document}